\documentclass[letterpaper]{article} 
\usepackage{aaai24}  
\usepackage{times}  
\usepackage{helvet}  
\usepackage{courier}  
\usepackage[hyphens]{url}  
\usepackage{graphicx} 
\usepackage{natbib}  
\usepackage{caption} 
\usepackage{algorithm}
\usepackage{algorithmic}
\usepackage{xcolor}
\usepackage{soul}
\usepackage{amsfonts}
\usepackage{amsmath}
\usepackage{amsthm}
\usepackage[final]{pdfpages}
\DeclareMathOperator*{\argmax}{arg\,max}

\newtheorem{thm}{Theorem}

\usepackage{newfloat}
\usepackage{listings}
\DeclareCaptionStyle{ruled}{labelfont=normalfont,labelsep=colon,strut=off} 
\floatstyle{ruled}
\newfloat{listing}{tb}{lst}{}
\floatname{listing}{Listing}
\title{Coevolutionary Algorithm for Building Robust Decision Trees under Minimax Regret}
\author{
  Adam {\.Z}ychowski \textsuperscript{\rm 1},
  Andrew Perrault \textsuperscript{\rm 2},
  Jacek Ma{\'n}dziuk \textsuperscript{\rm 1, 3, 4}
}
\affiliations{
  \textsuperscript{\rm 1}Faculty of Mathematics and Information Science, Warsaw University of Technology \\
  \textsuperscript{\rm 2}Department of Computer Science and Engineering, The Ohio State University\\
  \textsuperscript{\rm 3}Faculty of Computer Science, AGH University of Krakow \\
  \textsuperscript{\rm 4}Center of Excellence in Artificial Intelligence, AGH University of Krakow \\
  adam.zychowski@pw.edu.pl, perrault.17@osu.edu, jacek.mandziuk@pw.edu.pl
}

\begin{document}

\maketitle

\begin{abstract}
In recent years, there has been growing interest in developing robust machine learning (ML) models that can withstand adversarial attacks, including
one of the most widely adopted, efficient, and interpretable ML
algorithms---decision trees (DTs).
This paper proposes a novel coevolutionary algorithm (CoEvoRDT) designed to create robust 
DTs capable of handling noisy 
high-dimensional data in 
adversarial contexts. Motivated by the limitations of traditional 
DT algorithms, we leverage adaptive coevolution to allow 
DTs to evolve and learn from interactions with perturbed input data. CoEvoRDT 
alternately evolves competing populations of 
DTs and perturbed features, 
enabling 
construction of 
DTs with desired properties. 
CoEvoRDT is easily adaptable to various target metrics, allowing the use of tailored robustness criteria such as minimax regret. 
Furthermore, 
CoEvoRDT 
has potential to improve the results 
of other state-of-the-art methods by incorporating 
their outcomes (DTs they produce) into 
the initial population and 
optimize them in the process of coevolution. Inspired by the game theory, CoEvoRDT utilizes mixed Nash equilibrium to enhance convergence. The method is tested on 20 popular datasets and shows superior performance compared to 4 state-of-the-art 
algorithms. It outperformed all 
competing methods 
on 13 datasets 
with adversarial accuracy metrics, and 
on all 20 considered datasets 
with minimax regret. 
Strong experimental results and flexibility in choosing the error measure make CoEvoRDT a promising approach for constructing robust DTs in real-world 
applications.
\end{abstract}

\section{Introduction}

Decision trees (DTs) is
a popular, easily interpretable machine learning (ML) algorithm for classification and regression tasks. 
One of the primary challenges in 
DT construction 
is dealing with noisy and high-dimensional data. 
In particular, it has been shown that ML models (including DTs)
are vulnerable to adversarial,
perturbed 
samples that trick the model into misclassifying them~\cite{kantchelian2016evasion, zhang2020efficient, grosse2017statistical}. To address this challenge, researchers have proposed new defensive algorithms for creating \textit{robust} classification models (see, e.g., ~\citet{chakraborty2021survey}). A model is defined to be robust 
to some perturbation range of its input samples when it assigns the same class to all the samples within that perturbation range, so that 
small malicious alterations of input objects should not deceive a robust classifier.

The vast majority of defensive algorithms for DTs focus on adversarial accuracy~\cite{kantchelian2016evasion,chen2019robust,guo2022fast,ranzato2021genetic,justin2021optimal}. We argue that there are better metrics in principle, and the focus on adversarial accuracy has been driven by computational tractability. Adversarial accuracy is highly sensitive to accuracy on the worst-case perturbation, and when the perturbation range can be large, this can lead to a flattening of intuitively good models and bad ones, as the worst-case perturbations can ``defeat'' all models.

There are other metrics that can better evaluate model robustness, like \textit{max regret}~\cite{savage1951theory}. Max regret is defined as the maximum difference between the result of the given model and the result of the optimal model for any input data perturbation within a given range. Minimizing max regret is referred to as the minimax regret decision criterion.
Adversarial accuracy might provide an overly optimistic or pessimistic view of the model's robustness by focusing only on absolute accuracy value. In contrast, max regret is a more realistic approach since it counts the magnitude of the potential loss by considering the model trained on perturbed data. 
However, 
max regret cannot be directly optimized and 
used as a splitting criterion in the state-of-the-art algorithms.

In recent years, researchers successfully explored the potential of coevolutionary algorithms to various optimization problems~\cite{mahdavi2015metaheuristics} including 
DTs induction~\cite{aitkenhead2008co}. Coevolutionary algorithms 
consist in simultaneous evolution of multiple populations, 
each of them representing a different aspect of the problem. By fostering competition between populations, coevolutionary algorithms can guide the search 
towards the optimal solutions.

Considering the limitations of traditional 
DT algorithms and the promises of coevolutionary computation, we propose a novel coevolutionary algorithm specifically tailored for creating robust decision trees (RDTs) in adversarial contexts. 
Our approach leverages the power of adaptive coevolution, allowing to exploit the competitive interactions between populations of decision trees and adversarial perturbations to adapt and converge toward robust and accurate classifications for complex and noisy data. 
In this process, we can freely define robustness metrics to optimize (including max regret) which leads to the models better tailored to handle perturbed high-dimensional data. Because of the inherent flexibility of evolutionary methods, we can additionally integrate other objective criteria, such as fairness~\cite{aghaei2019learning, jo2023learning}.


\textbf{The main contribution} of this paper is proposition of a novel coevolutionary algorithm (CoEvoRDT) capable of creating robust decision trees. CoEvoRDT has the following key properties:
\begin{itemize}
  \item supremacy over state-of-the-art (SOTA) solution on 13 out of 20 datasets with adversarial accuracy metric and on all 20 datasets with minimax regret,
  \item predominance over existing evolutionary-based approaches for RDTs construction,
  \item to the best of our knowledge, it is the first algorithm able to directly optimize \textit{minimax regret} for 
  RDTs,
  \item it employs novel game theoretic approach for constructing the \textit{Hall of Fame} with \textit{Mixed Nash Equilibrium},
  \item the 
  algorithm is easily adaptable to various target metrics,
  \item by design, CoEvoRDT can be used for potential improvement of results 
  of other SOTA methods by including their resulting DTs
  in the initial population and optimizing them through coevolution.
\end{itemize}

\section{Problem definition}
Let $X \subset \mathbb{R}^d$ be a $d$-dimensional instance space (inputs) and $Y$ 
be the set of possible classes (outputs). 
A classical classification task is to find a function (model) $h:X \rightarrow Y$, $h(x_i) = y_i$, where $y_i$ is true class of $x_i$. Classification performance of model $h$ can be measured by accuracy: $$\mathrm{acc}(h) = \frac{1}{|X|}\sum_{x_i \in X} I[h(x_i) = y_i],$$ where $I[h(x_i) = y_i]$ returns 1 if $h$ predicts the true class of $x_i$,
and 0, otherwise.

Let $\mathcal{N}_\varepsilon (x) = \{z: ||z-x||_\infty \leq \varepsilon\}$ be a ball with center $x$ and radius $\varepsilon$ under the $L_\infty$ metric. The \textbf{adversarial accuracy} of a model $h$ is accuracy on the perturbation in the perturbation set that produces the lowest accuracy. It is formally defined as $$\mathrm{acc}_{\mathrm{adv}}(h, \epsilon) = \frac{1}{|X|}\sum_{x_i \in X} \min_{z _i\in \mathcal{N}_\varepsilon (x_i)} I[h(z_i) = y_i].$$

The \textbf{max regret} of a model $h$ is the maximum \emph{regret} among all possible perturbations $z \in \mathcal{N}_\varepsilon$. Regret is the difference between the best accuracy possible on a particular perturbation and the accuracy $h$ achieves: $$\mathrm{regret}(h,\{z_i\}) = \max_{h'} \mathrm{acc}(h',\{z_i\}) - \mathrm{acc}(h, \{z_i\}),$$ where $\mathrm{acc}(h, \{z_i\})$ is the accuracy achieved by $h$ when $\{x_i\}$ is replaced with $\{z_i\}$. Max regret be expressed as:
$$\mathrm{mr}(h) = \max_{z_i \in \mathcal{N}_\varepsilon (x_i)} \mathrm{regret}(h, \{z_i\})$$.

The problem 
addressed in this paper is finding a 
DT trained on $X$ that for a given $\varepsilon$ optimizes (maximizes for adversarial accuracy or minimizes for max regret) 
a given robustness metric (one of the two above-mentioned).

\section{Motivating example}

Consider a financial institution that 
makes loan acceptance decisions. DTs are well-suited for such high-stakes scenario~\cite{alaradi2020tree}. The dataset of loan applicants in Figure~\ref{fig:motivational-example} has two features: income $I$ and credit score $CS$. The system should correctly make a binary credit decision $D$: accept (1) or reject (0).

For the data in T1, a simple one-node DT can achieve 100\% accuracy. One possible decision rule to achieve this is $CS \geq 55$, which we call DT1.
\begin{figure}[ht]
\centering
\includegraphics[width=0.48\textwidth]
{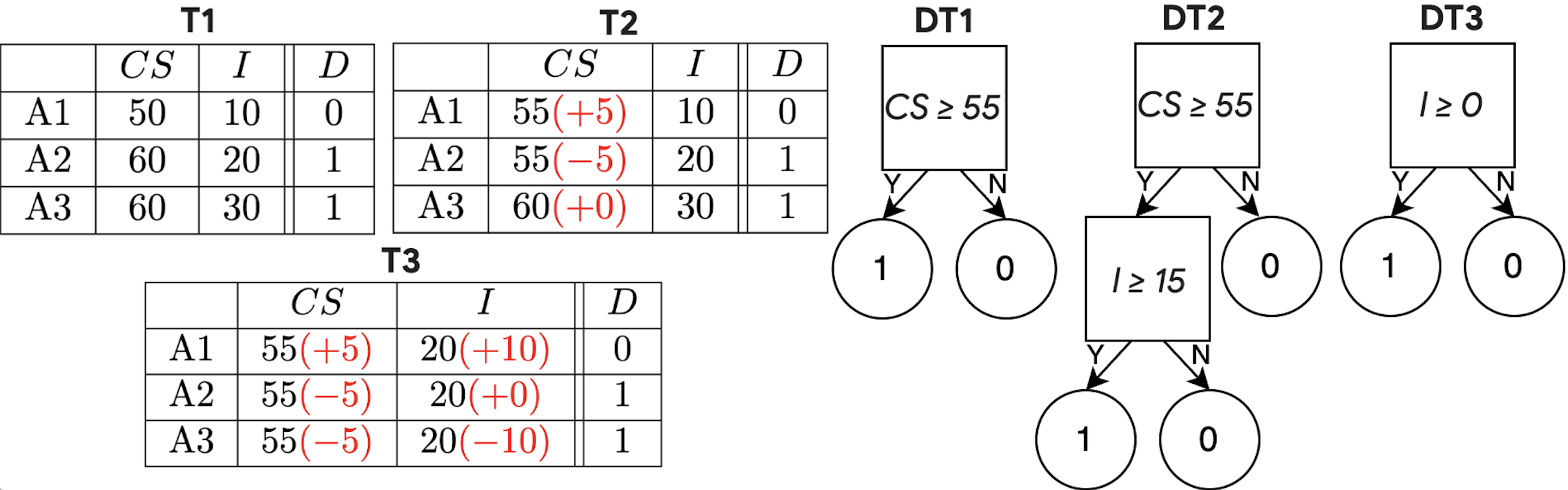}
\caption{Motivational example – perturbed input data and 3 decision trees.}
\label{fig:motivational-example}
\end{figure}

The features of training data may not be representative of test data due to bugs in the system, inaccurate input data, or distribution shift. Table T2 shows a potential perturbation of the data in T1.
In T2, DT1 misclassifies A1 (returning 1 instead of 0). A more robust decision tree, DT2,
accurately classifies all applicants in T1 and T2.

T3 is an example with a larger perturbation that affects both income and credit score. In this perturbation, the three applicants have the same features, meaning that no DT can classify them correctly. Thus, the adversarial accuracy against any perturbation set that includes T3 is at most $\frac{2}{3}$. 
However, achieving such accuracy is easy, any DT that always predicts 1 will do so, including the decision rule $I\geq 0$ (DT3). Consequently, for methods optimizing the adversarial accuracy metric, DT3 is one of the optimal solutions, but it is neither robust nor desired. Maximizing adversarial accuracy myopically focuses on the hardest perturbations in the perturbation set.

From a max regret perspective, DT2 outperforms DT3. Max regret considers not only the worst-case perturbation accuracy, but also the difference between the accuracy of the optimal DT and the robust DT for every perturbation. The regret of DT2 is 0 on all three datasets, resulting in a max regret of 0. DT3 achieves regret of $\frac{1}{3}$ on T1 and T2 and 0 on T3, resulting in a max regret of $\frac{1}{3}$. Thus, under the minimax regret criteria, DT2 would be selected over DT3. Adversarial accuracy loses its ability to distinguish between models as perturbations become large---intuitively good and bad can achieve the same scores.


\section{Related work}

There has been substantial recent work on the construction of robust decision trees.
One line of work aims to improve robustness by choosing more appropriate splitting criteria.
RIGDT-h~\cite{chen2019robust} constructs 
robust DTs based on the introduced notion of adversarial Gini impurity, a modification of classical Gini impurity~\cite{breiman2017classification} adapted to perturbed input data. This method was further improved 
in the GROOT algorithm~\cite{vos2021efficient}, which mimics the greedy recursive splitting strategy that traditional 
DTs use and scores splits with the adversarial Gini impurity. The most recent approach, Fast Provably Robust Decision Trees (FPRDT)~\cite{guo2022fast} is a greedy recursive approach to 
a direct minimization of the adversarial loss. It uses the 0/1 loss, rather than traditional surrogate losses such as square loss and softmax loss, as the splitting
criterion during the construction of the DT. In experiments, we compare to \citet{guo2022fast} and refer the reader to that paper for comparisons with prior methods. Max regret cannot be directly optimized by these methods, which leverage specific properties of adversarial accuracy to design splitting criteria.

\citet{ranzato2021genetic} introduced a genetic adversarial training algorithm (Meta-Silvae) to optimize 
DT stability, building on a history of using genetic algorithms for DTs~\citet{barros2011survey}, and leveraging the geometric of adversarial accuracy. Our approach utilizes a coevolutionary 
method and, to the best of our knowledge, it is the first application of this technique to creating robust DTs. At the same time, the effectiveness of coevolutionary algorithms was experimentally proven in many other domains including multi-objective optimization~\cite{meneghini2016competitive, tian2020coevolutionary}, non-cooperative games~\cite{razi2007finding, zychowski2023coevolution}, preventing adversarial attacks~\cite{o2018artificial}, or Generative Adversarial Networks training~\cite{costa2019coevolution, toutouh2023semi}.

An alternative, exact, approach to robust DTs proposed by \citet{justin2021optimal} uses a mixed-integer optimization formulation. However, at present, its present applicability is limited to small datasets (approximately less than 3200 samples and/or up to 36 features).



Although, to the best of our knowledge, max regret has not been specifically studied in
DTs, it was applied 
to other domains, e.g., neural network training~\cite{alaiz2007minimax}, reinforcement learning~\cite{azar2017minimax, xu2021robust}, robust planning in uncertain Markov decision processes~\cite{rigter2021minimax}, Security Games~\cite{nguyen2014regret}, and computing randomized Nash equilibrium~\cite{gilbert2017double}.

\section{CoEvoRDT algorithm}
A general overview of the \textit{Coevolutionary method for Robust Decision Trees} (CoEvoRDT) is presented in Algorithm~\ref{alg:pseudocode}. CoEvoRDT maintains two populations: one contains encoded 
DTs, and the other contains input data perturbations. 
Both populations are initialized with random elements and then developed alternately. First, the DT population is modified by evolutionary operators (crossover, mutation, and selection) through $l_c$ generations. Then, the perturbation
population is evolved through the same number of $l_c$ generations. The above loop is repeated until the stop condition is satisfied.

\begin{algorithm}[ht]
\caption{CoEvoRDT pseudocode.}\label{alg:pseudocode}
\begin{algorithmic}[1]
\small
\STATE $P_T \gets$ InitializeDecisionTreesPopulation()
\STATE $P_P \gets$ InitializePerturbationsPopulation()
\STATE HoF$_T$ = HoF$_P$ = $\emptyset$ // HoF - Hall of Fame
\STATE $N_\mathrm{top} = 20$

\STATE
\WHILE{stop condition not satisfied}

\FOR{1..$l_c$}
\STATE $P_T \gets P_T \; \cup$ Crossover($P_T$)
\STATE $P_T \gets P_T \; \cup$ Mutate($P_T$)
\STATE $P_T \gets$ Evaluate($P_T$, $P_P$, HoF$_P$)
\STATE $P^*_T \gets$ GetElite($P_T$)\\
\WHILE{$|P^*_T| < N_T$}
\STATE $P^*_T \gets P^*_T \; \cup $ BinaryTournament($P_T$)
\ENDWHILE
\STATE $P_T \gets P^*_T$
\STATE $\mathcal{T}, \mathcal{P} \gets$ MixedNashEquilibrium($P_T, P_P$)
\STATE HoF$_T \gets $ HoF$_T \cup \mathcal{T}$ 
\STATE HoF$_P \gets $ HoF$_P \cup \mathcal{P}$ 
\ENDFOR

\STATE
\FOR{1..$l_c$}
\STATE $P_P \gets P_P \; \cup$ Crossover($P_P$)
\STATE $P_P \gets P_P \; \cup$ Mutate($P_P$)
\STATE $P_P \gets$ Evaluate($P_P$, $P_T$, HoF$_T$, $N_\mathrm{top}$)
\STATE $P^*_P \gets$ GetElite($P_P$)\\
\WHILE{$|P^*_P| < N_P$}
\STATE $P^*_P \gets P^*_P \; \cup $ BinaryTournament($P_P$)
\ENDWHILE
\STATE $P_P \gets P^*_P$
\STATE $\mathcal{T}, \mathcal{P} \gets$ MixedNashEquilibrium($P_T, P_P$)
\STATE HoF$_T \gets $ HoF$_T \cup \mathcal{T}$ 
\STATE HoF$_P \gets $ HoF$_P \cup \mathcal{P}$ 
\ENDFOR

\ENDWHILE
\STATE
\STATE \textbf{return} $\argmax_{t \in P_T}\xi(t)$
\end{algorithmic}
\end{algorithm}

\subsection{Decision tree population}
The 
DT population contains $N_T$ individuals. Each individual represents one candidate solution 
(DT) which is encoded as a list of nodes. Each node is represented by a 7-tuple: $node = \{t, c, P, L, R, o, v, a\}$, where $t$ is a node number ($t = 0$ is the root node), $c$ is a class label of a terminal node (meaningful only for terminal nodes), $P$ is a pointer 
to the parent node, $L$ and $R$ are pointers to the left and right children, respectively (null 
in a terminal node), $o$ indicates which operator is to be used ($<,>,=$) and $v$ is a real number that indicates the value to be tested on attribute $a$.
 
The \textbf{initial population} consists of trees generated by randomly choosing attributes and split values, and halting the 
growth of each DT when the tree reaches a depth randomly selected from an interval $[2, 10]$.

Each individual from the population is selected for \textbf{crossover} with probability $p_c$. Selected individuals are paired randomly, and the crossover operator selects random nodes in two individuals and exchanges the entire subtrees corresponding to each selected node, generating two offspring individuals which are added to the current population.

The \textbf{mutation} operator introduces random changes to the individuals. Each individual is mutated with probability $p_m$. The mutation operator applies randomly one of the three following 
actions: (i) replacing a subtree with a randomly generated one, (ii) 
changing the information in a randomly selected node (setting a new random splitting value $v$ or operator $o$), (iii) prune a randomly selected subtree. For each mutated individual the mutation is applied $10$ times and the highest-fitness individual (among these 10) is added to the current population.

The \textbf{evaluation} procedure is performed against the perturbation population (described next). For each individual (a candidate 
DT) the metric being optimized 
is computed against all perturbations from the adversarial population. 
Since the number of perturbations is relatively small 
any arbitrary chosen metric can be effectively calculated and assigned 
as an individual's fitness value.

\subsection{Perturbation population}

The perturbation population consists of $N_P$ individuals $P \subset \mathbb{R}^d$. Each of them represents 
a perturbed input set $X$ (one perturbation per instance), i.e.,
$\forall_{x \in X}\exists!_{p_x \in P}: p_x \in \mathcal{N}_\varepsilon (x)$. 

The \textbf{initial population} contains random perturbations generated by drawing uniformly each perturbation element from the set of possible ones (according to $\varepsilon$ criteria).

The \textbf{crossover} procedure 
selects a random subset of individuals (each individual is taken with probability $p_c$) and pairs them randomly. Then, for each pair, perturbed input instances from both individuals 
are mixed randomly, i.e., given two crossed parents $c^0 = (x^0_1,\ldots,x^0_n)$ and $c^1 = (x^1_1,\ldots,x^1_n)$ 
the offspring is $\overline{c}^0 = (x^{i_1}_1,\ldots,x^{i_n}_n)$ and $\overline{c}^1 = (x^{1-i_1}_1,\ldots,x^{1-i_n}_n)$, where $i_1,\ldots,i_n \in \{0,1\}$.



\textbf{Mutation} is applied with probability $p_m$ independently to each individual. If a chromosome is selected for mutation, for each input instance and each attribute with probability 0.5 its encoded value is randomly perturbed, i.e. a new random feasible (according to $\varepsilon$ constraint) value is assigned.

\textbf{Evaluation} of the perturbation individuals is not 
an obvious task. 
On the one hand, assigning
the average accuracy
versus all 
DTs in the DT population as a fitness value may be a weak
approach. Observe that a given perturbation may be powerful only against a specific though relevant subset of the DTs,
and as such should be preserved, but averaging across all DTs
will decrease 
its fitness,
posing 
a risk of omitting it in the selection process. On the other hand, if the perturbation fitness value was computed only against the best DT
from the DT population, it would lead to an oscillation of the perturbation population. All perturbations would tend to be efficient for a particular 
DT, becoming vulnerable to other 
DTs and losing diversity. 
Thus, we use $N_\mathrm{top}=20$ highest-fitness 
DTs (merged with all DTs from Hall of Fame) to evaluate each perturbation and perform a targeted optimization with $N_\mathrm{top}=1$ only if we would otherwise terminate (see Stop condition). Experimental justification of the above 
choice is presented in the supplementary material (SM).

\subsection{Hall of Fame}
Hall of Fame (HoF) is a mechanism used to retain and store the best-performing individuals (solutions) that have been encountered during the evolutionary process. By preserving them, the HoF prevents the loss of valuable information and ensures that the best-performing solutions are not discarded during the evolution. 
The most common approach is to add one of the highest-fitness individuals from each generation~\cite{michalewicz1996gas}. We find this approach 
to suboptimal with respect to diversity.
Although the HoF stores the best solutions, it can also be used to maintain a diverse set of high-performing individuals. Diversity is essential in evolutionary algorithms to avoid premature convergence, 
when the algorithm gets stuck in a local optimum and fails to explore better solutions. The HoF can promote diversity by storing solutions that represent different regions of the solution space. 

In our coevolutionary approach, HoF is used to assess solutions more accurately. Namely, instead of calculating the fitness function only against individuals from the adversarial population, it is calculated against a merged set of HoF and population individuals.

Instead of adding the highest-fitness individual to the HoF, in CoEvoRDT, we use a game-theoretic approach. Decision trees and perturbation populations can be treated as sets of strategies of two players in 
a non-cooperative zero-sum game. Then, it is possible to calculate mixed Nash equilibrium. The result is the pair of mixed strategies, i.e., a subset of 
DTs from the population with assigned probabilities $\mathcal{T} = \{(T_1, p_{T_1}), \ldots, (T_n, p_{T_n})\}$ and a similar subset of perturbations with probabilities $\mathcal{P} = \{(P_1, p_{P_1}), \ldots, (P_m, p_{P_m})\}$. Formally, a mixed Nash equilibrium is a pair $(\mathcal{T},\mathcal{P})$ such as $\forall_{\mathcal{T'} \neq \mathcal{T}} \xi_T(\mathcal{T'},\mathcal{P}) \preceq \xi_T(\mathcal{T},\mathcal{P})$ and $\forall_{\mathcal{P'} \neq \mathcal{P}} \xi_P(\mathcal{T},\mathcal{P'}) \preceq \xi_P(\mathcal{T},\mathcal{P})$ where $\xi_{T|P}(\mathcal{T},\mathcal{P})$ denotes some objective robustness metrics calculated for a ``mixed'' decision tree $\mathcal{T}$ and ''mixed'' perturbation $\mathcal{P}$ (either adversarial accuracy or max regret, in our experiments). Note that this is zero-sum because, in robust optimization, the adversary aims to minimize the DT payoff (i.e., objective function): $\xi_{T}(\mathcal{T},\mathcal{P}) = -\xi_{P}(\mathcal{T},\mathcal{P})$. 
We add mixed strategies from Nash equilibria to both HoFs, and they are used in the evaluation process as described above. To evaluate a metric against a mixed object (tree or perturbation), we calculate the expected metric value---first computing the metric for each pair of pure strategies and then taking a weighted average according to the Nash equilibrium probabilities.
We limit HoF size by the lowest-fitness element when a fixed maximum size is exceeded.

A similar approach was previously proposed in~\cite{ficici2003game} but instead of storing in HoF pure strategies from mixed Nash equilibrium we add mixed strategies. The intuition is that a mixed tree is more robust to 
diverse perturbations, which has a positive 
impact on the evolution of perturbations. Similarly, mixed perturbations force the DT population to create more robust DTs that are resistant to a wide spectrum of perturbed data. We demonstrate this in the experiments section.



\subsection{Selection}
The selection process decides which individuals from the current population will be promoted to the next generation. In the beginning, $e$ individuals with the highest fitness value are unconditionally transferred to the next generation. They are called \textit{elite} and preserve the highest-fitness solutions.
Then, a \textit{binary tournament} is repeatedly executed until the next generation population is filled with $N$ individuals. 
In each tournament, two individuals are sampled (with replacement) from the current population (including those affected by crossover and/or mutation). The higher-fitness chromosome (the winner) is promoted to the next generation with probability $p_s$ (so-called selection pressure parameter). Otherwise, the lower-fitness one is promoted.

\subsection{Stop condition}

The algorithm ends when at least one of the following conditions is satisfied: (a) CoEvoRDT attains the maximum number of generations ($l_g$), (b) no improvement of the best-found solution 
(DT) is observed in consecutive $l_c$ generations.
If condition (b) is satisfied, an additional local perturbation improvement subroutine is performed. This procedure is part of the stopping condition and aims to find a better perturbation for the best-fitness decision trees (DTs). Specifically, for each DT with the highest fitness value, the perturbation population evolves using the same process as outlined in lines 21--33 of the Algorithm~\ref{alg:pseudocode}, but with $N_\mathrm{top}=1$ (see line 24). This means that the evaluation for each perturbation is conducted against the DT with current highest fitness. If, for all of those DTs, this routine discovers a perturbation that decreases the fitness of the DT, the counter $l_c$ is reset to zero, and the algorithm execution continues (with $N_\mathrm{top}=20$ and the population's state before local perturbation improvement subroutine execution supplemented with the newly found better perturbation).

To verify conditions (a) and (b) only generations of the DT population are considered. 
The highest-fitness 
DT is returned as a CoEvoRDT result.

\subsection{Convergence}
The alternating optimization can be understood as improving candidate DTs while progressively tightening a bound on the robust objective, as any set of perturbation provides a upper (resp., lower bound) on adversarial accuracy (resp., max regret). When $N_\mathrm{top}=1$ (e.g., at convergence), the bound of the objective has been tightened as much as possible for a candidate DT.

\begin{thm}
If both the decision tree and the perturbation population contain an individual that maximizes their fitness against the opposing population, the decision tree in the current population with the highest fitness optimizes the robust objective.
\end{thm}
\begin{proof}
Suppose the stop condition is met and fitness is maximized by the decision tree and perturbation populations. We claim that the highest-fitness decision tree $h$ maximizes the robust objective. Suppose that this is not the case. Then, either (i) there exists a perturbation $z$ that would lower $h$'s fitness and is missing from the adversarial population or (ii) there is a decision tree $h'$ that would have higher fitness than $h$ but is missing. (ii) cannot occur because we assume that $h$ maximizes fitness. (i) cannot happen because $N_\mathrm{top}=1$ when the stop condition is met.
\end{proof}


\section{Results and discussion}

\textbf{Experimental setup.} The proposed method was tested on 20 widely used classification benchmark problems
of 
various characteristics – 
the number of instances, features, and perturbation coefficient. All selected datasets were 
used in previous studies mentioned in the Related work section, and they are publicly available at https://www.openml.org. 
Table~\ref{tab:datasets} summarizes their basic parameters.
%
\begin{table}[ht]
\scriptsize
\centering
\begin{tabular}{|l|c|c|c|c|}
\hline
\textbf{dataset}  & $\varepsilon$ & Instances & Features & Classes \\ \hline \vspace {-1 mm}
ionos & 0.2  & 351 & 34 & 2 \\ 
\vspace {-1 mm}
breast  & 0.3  & 683 & 9  & 2 \\ \vspace {-1 mm}
diabetes  & 0.05 & 768 & 8  & 2 \\ \vspace {-1 mm}
bank  & 0.1  & 1372  & 4  & 2 \\ \vspace {-1 mm}
Japan:3v4  & 0.1  & 3087  & 14 & 2 \\ \vspace {-1 mm}
spam  & 0.05 & 4601  & 57 & 2 \\ \vspace {-1 mm}
GesDvP  & 0.01 & 4838  & 32 & 2 \\ \vspace {-1 mm}
har1v2  & 0.1  & 3266  & 561  & 2 \\ \vspace {-1 mm}
wine  & 0.1  & 6497  & 11 & 2 \\ \vspace {-1 mm}
collision-det & 0.1  & 33000 & 6  & 2 \\ \vspace {-1 mm}
mnist:1v5 & 0.3  & 13866 & 784  & 2 \\ \vspace {-1 mm}
mnist:2v6 & 0.3  & 13866 & 784  & 2 \\ \vspace {-1 mm}
mnist & 0.3  & 70000 & 784  & 10  \\ \vspace {-1 mm}
f-mnist:2v5  & 0.2  & 14000 & 784  & 2 \\ \vspace {-1 mm}
f-mnist:3v4  & 0.2  & 14000 & 784  & 2 \\ \vspace {-1 mm}
f-mnist:7v9  & 0.2  & 14000 & 784  & 2 \\ \vspace {-1 mm}
f-mnist & 0.2  & 70000 & 784  & 10  \\ \vspace {-1 mm}
cifar10:0v5 & 0.1  & 12000 & 3072 & 2 \\ \vspace {-1 mm}
cifar10:0v6 & 0.1  & 12000 & 3072 & 2 \\ 
cifar10:4v8 & 0.1  & 12000 & 3072 & 2 \\ \hline
\end{tabular}
\caption{Basic parameters of the benchmark datasets.}
\label{tab:datasets}
\end{table}

The CoEvoRDT parameter values used in the experiments and their selection process is described in detail in the SM. 
Since there is no straightforward method to calculate the exact values of adversarial accuracy and minimax regret (due to the presence of infinitely many possible perturbations), the results presented below are computed based on a sample of $10^5$ random perturbations (the same set for each compared method). The reasoning behind choosing this particular sample size is explained in the SM.

We adopted the Lemke-Howson algorithm~\cite{lemke1964equilibrium} for calculating Mixed Nash Equilibrium from Nashpy Python library~\cite{nashpy}. The CART~\cite{breiman2017classification} method was used for computing the reference tree for minimax regret (i.e., highest accuracy trees for a particular perturbation). Statistical significance was checked according to the paired $t$-test  with $p$-value $\leq 0.05$. All tests were run on Intel Xeon Silver 4116 @ 2.10GHz. CoEvoRDT source code is 
made publicly available at
https://github.com/zychowskia/CoEvoRDT.


\begin{table*}[htp]
\centering
\scriptsize
\begin{tabular}{|l|c|c|c|c|c|c||c|}
\hline
\textbf{dataset} & CART  & Meta Silvae  & RIGDT-h  & GROOT  & FPRDT  & CoEvoRDT  & \textbf{CoEvoRDT+FPRDT} \\ \hline \vspace {-1 mm}
ionos  & .094$\pm$.000 & .075$\pm$.007 & .071$\pm$.006 & .061$\pm$.005 & .061$\pm$.006 & \colorbox{gray!20}{\textbf{.052$\pm$.004}} & .052$\pm$.005 \\ \vspace {-1 mm}
breast  & .103$\pm$.000 & .056$\pm$.006 & .069$\pm$.006 & .059$\pm$.005 & .057$\pm$.005 & \colorbox{gray!20}{\textbf{.049$\pm$.004}} & .049$\pm$.005 \\ \vspace {-1 mm}
diabetes  & .202$\pm$.000 & .126$\pm$.008 & .132$\pm$.009 & .124$\pm$.009 & .117$\pm$.007 & \colorbox{gray!20}{\textbf{.096$\pm$.006}} & .094$\pm$.007 \\ \vspace {-1 mm}
bank  & .186$\pm$.000 & .102$\pm$.007 & .108$\pm$.008 & .090$\pm$.006  & .089$\pm$.007 & \colorbox{gray!20}{\textbf{.076$\pm$.006}} & .076$\pm$.006 \\ \vspace {-1 mm}
Japan3v4  & .107$\pm$.000 & .090$\pm$.006  & .083$\pm$.006 & .067$\pm$.006 & .066$\pm$.004 & \colorbox{gray!20}{\textbf{.062$\pm$.006}} & .061$\pm$.006 \\ \vspace {-1 mm}
spam  & .097$\pm$.000 & .079$\pm$.006 & .083$\pm$.006 & .074$\pm$.006 & .074$\pm$.006 & \colorbox{gray!20}{\textbf{.070$\pm$.005}}  & .069$\pm$.005 \\ \vspace {-1 mm}
GesDvP  & .152$\pm$.000 & .129$\pm$.008 & .133$\pm$.010 & .129$\pm$.008 & .131$\pm$.009 & \colorbox{gray!20}{\textbf{.114$\pm$.007}} & .114$\pm$.007 \\ \vspace {-1 mm}
har1v2  & .105$\pm$.000 & .074$\pm$.006 & .084$\pm$.007 & .068$\pm$.006 & .068$\pm$.006 & \colorbox{gray!20}{\textbf{.064$\pm$.005}} & .064$\pm$.005 \\ \vspace {-1 mm}
wine  & .140$\pm$.000  & .125$\pm$.008 & .127$\pm$.009 & .111$\pm$.009 & .109$\pm$.008 & \colorbox{gray!20}{\textbf{.090$\pm$.006}}  & .090$\pm$.007  \\ \vspace {-1 mm}
collision-det  & .142$\pm$.000 & .099$\pm$.007 & .093$\pm$.007 & .088$\pm$.006 & .091$\pm$.007 & \colorbox{gray!20}{\textbf{.061$\pm$.006}} & .059$\pm$.006 \\ \vspace {-1 mm}
mnist:1v5  & .249$\pm$.000 & .078$\pm$.007 & .076$\pm$.006 & .071$\pm$.006 & .067$\pm$.005 & \colorbox{gray!20}{\textbf{.055$\pm$.006}} & .055$\pm$.005 \\ \vspace {-1 mm}
mnist:2v6  & .268$\pm$.000 & .083$\pm$.007 & .087$\pm$.006 & .072$\pm$.005 & .069$\pm$.005 & \colorbox{gray!20}{\textbf{.055$\pm$.004}} & .054$\pm$.004 \\ \vspace {-1 mm}
mnist  & .395$\pm$.000 & .143$\pm$.009 & .139$\pm$.009 & .125$\pm$.007 & .124$\pm$.009 & \colorbox{gray!20}{\textbf{.113$\pm$.008}} & .112$\pm$.008 \\ \vspace {-1 mm}
f-mnist2v5  & .273$\pm$.000 & .254$\pm$.015 & .249$\pm$.015 & .223$\pm$.013 & .238$\pm$.014 & \colorbox{gray!20}{\textbf{.196$\pm$.011}} & .196$\pm$.011 \\ \vspace {-1 mm}
f-mnist3v4  & .290$\pm$.000  & .259$\pm$.014 & .254$\pm$.015 & .246$\pm$.014 & .232$\pm$.013 & \colorbox{gray!20}{\textbf{.202$\pm$.011}} & .199$\pm$.011 \\ \vspace {-1 mm}
f-mnist7v9  & .283$\pm$.000 & .255$\pm$.014 & .251$\pm$.015 & .237$\pm$.014 & .240$\pm$.014  & \colorbox{gray!20}{\textbf{.208$\pm$.013}} & .207$\pm$.012 \\ \vspace {-1 mm}
f-mnist  & .427$\pm$.000 & .345$\pm$.020 & .337$\pm$.018 & .292$\pm$.017 & .286$\pm$.016 & \colorbox{gray!20}{\textbf{.238$\pm$.014}} & .237$\pm$.015 \\ \vspace {-1 mm}
cifar10:0v5  & .419$\pm$.000 & .351$\pm$.019 & .379$\pm$.021 & .347$\pm$.019 & .314$\pm$.018 & \colorbox{gray!20}{\textbf{.241$\pm$.015}} & .236$\pm$.013 \\ \vspace {-1 mm}
cifar10:0v6  & .403$\pm$.000 & .362$\pm$.021 & .368$\pm$.020 & .342$\pm$.018 & .341$\pm$.019 & \colorbox{gray!20}{\textbf{.289$\pm$.016}} & .289$\pm$.016 \\
cifar10:4v8  & .408$\pm$.000 & .357$\pm$.019 & .360$\pm$.021  & .339$\pm$.018 & .331$\pm$.019 & \colorbox{gray!20}{\textbf{.283$\pm$.016}} &  .281$\pm$.017 \\ \hline
\end{tabular}
\caption{
\textbf{Max regrets} (mean $\pm$ std error). CoEvoRDT+FPRDT obtained the best results for all datasets. The best results (except CoEvoRDT+FPRDT) are \textbf{bolded}. \colorbox{gray!20}{Gray background} indicates that a given method is statistically significantly better than all other methods (except CoEvoRDT+FPRDT). }
\label{tab:results-max-regret}
\end{table*}
\begin{table*}[p]
\centering
\scriptsize
\begin{tabular}{|l|c|c|c|c|c|c||c|}
\hline
\textbf{dataset} & CART  & Meta Silvae  & RIGDT-h  & GROOT  & FPRDT  & CoEvoRDT  & \textbf{CoEvoRDT+FPRDT} \\ \hline \vspace {-1 mm}
ionos & .310$\pm$.000  & .695$\pm$.039  & .701$\pm$.045  & .783$\pm$.047  & \colorbox{gray!20}{\textbf{.795$\pm$.047}}  & .791$\pm$.044 & .795$\pm$.049 \\ \vspace {-1 mm}
breast  & .250$\pm$.000  & .797$\pm$.047  & .838$\pm$.052  & .874$\pm$.047  & .876$\pm$.055  & \colorbox{gray!20}{\textbf{.885$\pm$.054}} & .889$\pm$.056 \\ \vspace {-1 mm}
diabetes  & .542$\pm$.000 & .554$\pm$.035  & .569$\pm$.033  & .623$\pm$.043  & \colorbox{gray!20}{\textbf{.648$\pm$.039}}  & .617$\pm$.038 & .648$\pm$.037 \\ \vspace {-1 mm}
bank  & .633$\pm$.000 & .510$\pm$.031 & .468$\pm$.033  & .541$\pm$.036  & \textbf{.658$\pm$.040}  & .657$\pm$.043 & \fbox{.663$\pm$.037} \\ \vspace {-1 mm}
Japan3v4  & .576$\pm$.000 & .566$\pm$.035  & .564$\pm$.037  & .584$\pm$.035  & \textbf{.667$\pm$.039}  & .665$\pm$.037 & .668$\pm$.037 \\ \vspace {-1 mm}
spam  & .302$\pm$.000 & .637$\pm$.036  & .467$\pm$.028  & .723$\pm$.045  & .746$\pm$.049  & \textbf{.751$\pm$.049} & .753$\pm$.045 \\ \vspace {-1 mm}
GesDvP  & .478$\pm$.000 & .637$\pm$.039  & .548$\pm$.033  & .716$\pm$.045  & .735$\pm$.040  & \textbf{.740$\pm$.046}  & .741$\pm$.044 \\ \vspace {-1 mm}
har1v2  & .232$\pm$.000 & .706$\pm$.045  & .707$\pm$.047  & .806$\pm$.048  & .804$\pm$.049  & \colorbox{gray!20}{\textbf{.818$\pm$.054}} & .820$\pm$.052  \\ \vspace {-1 mm}
wine  & .620$\pm$.000  & .637$\pm$.039  & .474$\pm$.027  & .637$\pm$.036  & .674$\pm$.037  & \colorbox{gray!20}{\textbf{.688$\pm$.046}} & \fbox{.692$\pm$.047} \\\vspace {-1 mm}
collision-det & .743$\pm$.000 & .772$\pm$.047  & .764$\pm$.044  & .784$\pm$.052  & .792$\pm$.051  & \textbf{.798$\pm$.053} & \fbox{.803$\pm$.049} \\ \vspace {-1 mm}
mnist:1v5 & .921$\pm$.000 & .952$\pm$.056  & .957$\pm$.054  & .954$\pm$.056  & \textbf{.966$\pm$.058}  & .964$\pm$.059 & .969$\pm$.061 \\ \vspace {-1 mm}
mnist:2v6 & .862$\pm$.000 & .906$\pm$.054  & .919$\pm$.050  & .917$\pm$.052  & \colorbox{gray!20}{\textbf{.922$\pm$.049}}  & .917$\pm$.053 & .922$\pm$.051 \\ \vspace {-1 mm} 
mnist & .673$\pm$.000 & .702$\pm$.041  & .704$\pm$.042  & .743$\pm$.048  & .742$\pm$.049  & \textbf{.745$\pm$.043} & \fbox{.754$\pm$.046} \\ \vspace {-1 mm}
f-mnist2v5  & .675$\pm$.000 & .951$\pm$.053  & .945$\pm$.060  & .971$\pm$.057  & .978$\pm$.055  & \textbf{.982$\pm$.055} & .982$\pm$.059 \\ \vspace {-1 mm}
f-mnist3v4  & .632$\pm$.000 & .808$\pm$.049  & .793$\pm$.044  & .819$\pm$.048  & .865$\pm$.050  & \textbf{.869$\pm$.056} & .870$\pm$.054  \\ \vspace {-1 mm}
f-mnist7v9  & .642$\pm$.000 & .824$\pm$.045  & .81$\pm$.052 & .829$\pm$.052  & \textbf{.876$\pm$.050}  & .868$\pm$.054 & \fbox{.880$\pm$.047}  \\ \vspace {-1 mm}
f-mnist & .464$\pm$.000 & .492$\pm$.033  & .525$\pm$.033  & .536$\pm$.035  & .531$\pm$.033  & \colorbox{gray!20}{\textbf{.544$\pm$.036}} & .546$\pm$.040 \\ \vspace {-1 mm}
cifar10:0v5 & .296$\pm$.000 & .502$\pm$.033  & .347$\pm$.026  & .485$\pm$.036  & .678$\pm$.046  & \colorbox{gray!20}{\textbf{.685$\pm$.039}} & \fbox{.693$\pm$.039} \\ \vspace {-1 mm}
cifar10:0v6 & .587$\pm$.000 & .540$\pm$.038 & .477$\pm$.029  & .556$\pm$.037  & .688$\pm$.040  & \textbf{.692$\pm$.046} & \fbox{.697$\pm$.043} \\ 
cifar10:4v8 & .256$\pm$.000 & .514$\pm$.032  & .488$\pm$.033  & .473$\pm$.032  & .661$\pm$.042  & \textbf{.663$\pm$.045} & .664$\pm$.037 \\ \hline
\end{tabular}
\caption{
\textbf{Adversarial accuracies} (mean $\pm$ std error). CoEvoRDT+FPRDT obtained the best results for all datasets. \fbox{Box} denotes that CoEvoRDT+FPRDT is statistically significantly better than all other methods. The best results (except CoEvoRDT+FPRDT) are \textbf{bolded}. \colorbox{gray!20}{Gray background} indicates that a given method is statistically significantly better than all other methods (except CoEvoRDT+FPRDT). }
\label{tab:results-adversarial-accuracy}
\end{table*}

\begin{table*}[htp]
\centering
\def\arraystretch{1.02}
\setlength{\tabcolsep}{0.2em}
\resizebox{\textwidth}{!}{
\begin{tabular}{|c|cccc||cccc||cccc|}
\hline
  & \multicolumn{4}{c||}{minimax regret} & \multicolumn{4}{c||}{adversarial accuracy} & \multicolumn{4}{c|}{computation time {[}s{]}} \\ \hline
N & \multicolumn{1}{c|}{N FPRDT} & \multicolumn{1}{c|}{N CoEvoRDT} & \multicolumn{1}{c|}{\begin{tabular}[c]{@{}c@{}}CoEvoRDT \\ + N FPRDT\end{tabular}} & \begin{tabular}[c]{@{}c@{}}N CoEvoRDT \\ + N FPRDT\end{tabular} & \multicolumn{1}{c|}{N FPRDT} & \multicolumn{1}{c|}{N CoEvoRDT} & \multicolumn{1}{c|}{\begin{tabular}[c]{@{}c@{}}CoEvoRDT \\ + N FPRDT\end{tabular}} & \begin{tabular}[c]{@{}c@{}}N CoEvoRDT \\ + N FPRDT\end{tabular} & \multicolumn{1}{c|}{N FPRDT} & \multicolumn{1}{c|}{N CoEvoRDT} & \multicolumn{1}{c|}{\begin{tabular}[c]{@{}c@{}}CoEvoRDT \\ + N FPRDT\end{tabular}} & \begin{tabular}[c]{@{}c@{}}N CoEvoRDT \\ + N FPRDT\end{tabular} \\ \hline
1 & \multicolumn{1}{c|}{.304} & \multicolumn{1}{c|}{.238}  & \multicolumn{1}{c|}{.237} & .237 & \multicolumn{1}{c|}{.531} & \multicolumn{1}{c|}{.544} & \multicolumn{1}{c|}{.546} & .546 & \multicolumn{1}{c|}{19}  & \multicolumn{1}{c|}{79} & \multicolumn{1}{c|}{97}  & 97  \\ \hline
2 & \multicolumn{1}{c|}{.302} & \multicolumn{1}{c|}{.237}  & \multicolumn{1}{c|}{.236} & .236 & \multicolumn{1}{c|}{.535} & \multicolumn{1}{c|}{.546}  & \multicolumn{1}{c|}{.548} & .548 & \multicolumn{1}{c|}{40}  & \multicolumn{1}{c|}{161}  & \multicolumn{1}{c|}{114} & 185 \\ \hline
3 & \multicolumn{1}{c|}{.301} & \multicolumn{1}{c|}{.237}  & \multicolumn{1}{c|}{.236} & .236 & \multicolumn{1}{c|}{.539} & \multicolumn{1}{c|}{.548}  & \multicolumn{1}{c|}{.549} & .550 & \multicolumn{1}{c|}{60}  & \multicolumn{1}{c|}{240}  & \multicolumn{1}{c|}{134} & 272 \\ \hline
4 & \multicolumn{1}{c|}{.300} & \multicolumn{1}{c|}{.236}  & \multicolumn{1}{c|}{.236} & .235 & \multicolumn{1}{c|}{.545} & \multicolumn{1}{c|}{.550}  & \multicolumn{1}{c|}{.552} & .553 & \multicolumn{1}{c|}{80}  & \multicolumn{1}{c|}{321}  & \multicolumn{1}{c|}{165} & 362 \\ \hline
5 & \multicolumn{1}{c|}{.299} & \multicolumn{1}{c|}{.236}  & \multicolumn{1}{c|}{.235} & .235 & \multicolumn{1}{c|}{.548} & \multicolumn{1}{c|}{.552}  & \multicolumn{1}{c|}{.554} & .557 & \multicolumn{1}{c|}{99}  & \multicolumn{1}{c|}{406}  & \multicolumn{1}{c|}{183} & 496 \\ \hline
10  & \multicolumn{1}{c|}{.293} & \multicolumn{1}{c|}{.234}  & \multicolumn{1}{c|}{.233} & .233 & \multicolumn{1}{c|}{.552} & \multicolumn{1}{c|}{.557}  & \multicolumn{1}{c|}{.559} & .562 & \multicolumn{1}{c|}{195} & \multicolumn{1}{c|}{774}  & \multicolumn{1}{c|}{264} & 939 \\ \hline
20  & \multicolumn{1}{c|}{.284} & \multicolumn{1}{c|}{.230}  & \multicolumn{1}{c|}{.230} & .229 & \multicolumn{1}{c|}{.558} & \multicolumn{1}{c|}{.564}  & \multicolumn{1}{c|}{.566} & .568 & \multicolumn{1}{c|}{391} & \multicolumn{1}{c|}{1553} & \multicolumn{1}{c|}{454} & 1869  \\ \hline
50  & \multicolumn{1}{c|}{.282} & \multicolumn{1}{c|}{.229}  & \multicolumn{1}{c|}{.229} & .227 & \multicolumn{1}{c|}{.563} & \multicolumn{1}{c|}{.568}  & \multicolumn{1}{c|}{.568} & .569 & \multicolumn{1}{c|}{956} & \multicolumn{1}{c|}{3863} & \multicolumn{1}{c|}{999} & 4564  \\ \hline
100 & \multicolumn{1}{c|}{.282} & \multicolumn{1}{c|}{.228}  & \multicolumn{1}{c|}{.229} & .227 & \multicolumn{1}{c|}{.566} & \multicolumn{1}{c|}{.568}  & \multicolumn{1}{c|}{.568} & .569 & \multicolumn{1}{c|}{1921}  & \multicolumn{1}{c|}{7981} & \multicolumn{1}{c|}{1990}  & 9026  \\ \hline
\end{tabular}
}
\caption{
Best results of repeated $N$ algorithms' runs for \textbf{fashion-mnist} dataset. In CoEvoRDT + N FPRDT 
output 
DTs from N FPRDT independent runs were incorporated into CoEvoRDT initial population.}
\label{tab:multiple-runs-fashion-mnist}
\end{table*}

\begin{table*}[h!t!]
\centering
\def\arraystretch{1.02}
\setlength{\tabcolsep}{0.2em}
\resizebox{\textwidth}{!}{
\begin{tabular}{|c|ccccc||ccccc||ccccc|}
\hline
 & \multicolumn{5}{c||}{minimax regret}  & \multicolumn{5}{c||}{adversarial accuracy}  & \multicolumn{5}{c|}{computation time [s]}   \\ \hline
HoF size & \multicolumn{1}{c|}{\begin{tabular}[c]{@{}c@{}}Nash mixed \\ tree\end{tabular}} & \multicolumn{1}{c|}{\begin{tabular}[c]{@{}c@{}}Top K as \\ mixed tree\end{tabular}} & \multicolumn{1}{c|}{\begin{tabular}[c]{@{}c@{}}Nash single \\ trees\end{tabular}}  & \multicolumn{1}{c|}{Top K} & \begin{tabular}[c]{@{}c@{}}Best\end{tabular} & \multicolumn{1}{c|}{\begin{tabular}[c]{@{}c@{}}Nash mixed \\ tree\end{tabular}} & \multicolumn{1}{c|}{\begin{tabular}[c]{@{}c@{}}Top K as \\ mixed tree\end{tabular}} & \multicolumn{1}{c|}{\begin{tabular}[c]{@{}c@{}}Nash single \\ trees\end{tabular}}  & \multicolumn{1}{c|}{Top K} & \begin{tabular}[c]{@{}c@{}}Best\end{tabular} & \multicolumn{1}{c|}{\begin{tabular}[c]{@{}c@{}}Nash mixed \\ tree\end{tabular}} & \multicolumn{1}{c|}{\begin{tabular}[c]{@{}c@{}}Top K as \\ mixed tree\end{tabular}} & \multicolumn{1}{c|}{\begin{tabular}[c]{@{}c@{}}Nash single \\ trees\end{tabular}}  & \multicolumn{1}{c|}{Top K} & \begin{tabular}[c]{@{}c@{}}Best\end{tabular} \\ \hline
0  & \multicolumn{1}{c|}{.261}& \multicolumn{1}{c|}{.261}& \multicolumn{1}{c|}{.261} & \multicolumn{1}{c|}{.261} & .261 & \multicolumn{1}{c|}{.533}& \multicolumn{1}{c|}{.533}& \multicolumn{1}{c|}{.533} & \multicolumn{1}{c|}{.533} & .533 & \multicolumn{1}{c|}{47}& \multicolumn{1}{c|}{47}& \multicolumn{1}{c|}{47}& \multicolumn{1}{c|}{47} & 47 \\ \hline
10 & \multicolumn{1}{c|}{.242}& \multicolumn{1}{c|}{.248}& \multicolumn{1}{c|}{.247} & \multicolumn{1}{c|}{.251} & .259 & \multicolumn{1}{c|}{.535}& \multicolumn{1}{c|}{.535}& \multicolumn{1}{c|}{.535} & \multicolumn{1}{c|}{.534} & .533 & \multicolumn{1}{c|}{50}& \multicolumn{1}{c|}{50}& \multicolumn{1}{c|}{50}& \multicolumn{1}{c|}{50} & 50 \\ \hline
20 & \multicolumn{1}{c|}{.240}& \multicolumn{1}{c|}{.246}& \multicolumn{1}{c|}{.245} & \multicolumn{1}{c|}{.249} & .256 & \multicolumn{1}{c|}{.536}& \multicolumn{1}{c|}{.536}& \multicolumn{1}{c|}{.536} & \multicolumn{1}{c|}{.536} & .534 & \multicolumn{1}{c|}{55}& \multicolumn{1}{c|}{54}& \multicolumn{1}{c|}{55}& \multicolumn{1}{c|}{56} & 51 \\ \hline
50 & \multicolumn{1}{c|}{.241}& \multicolumn{1}{c|}{.244}& \multicolumn{1}{c|}{.245} & \multicolumn{1}{c|}{.249} & .254 & \multicolumn{1}{c|}{.536}& \multicolumn{1}{c|}{.536}& \multicolumn{1}{c|}{.536} & \multicolumn{1}{c|}{.536} & .534 & \multicolumn{1}{c|}{61}& \multicolumn{1}{c|}{58}& \multicolumn{1}{c|}{59}& \multicolumn{1}{c|}{62} & 54 \\ \hline
100& \multicolumn{1}{c|}{.239}& \multicolumn{1}{c|}{.243}& \multicolumn{1}{c|}{.243} & \multicolumn{1}{c|}{.247} & .253 & \multicolumn{1}{c|}{.538}& \multicolumn{1}{c|}{.538}& \multicolumn{1}{c|}{.537} & \multicolumn{1}{c|}{.537} & .535 & \multicolumn{1}{c|}{68}& \multicolumn{1}{c|}{63}& \multicolumn{1}{c|}{66}& \multicolumn{1}{c|}{65} & 56 \\ \hline
200& \multicolumn{1}{c|}{.238}& \multicolumn{1}{c|}{.242}& \multicolumn{1}{c|}{.242} & \multicolumn{1}{c|}{.244} & .250 & \multicolumn{1}{c|}{.543}& \multicolumn{1}{c|}{.539}& \multicolumn{1}{c|}{.540} & \multicolumn{1}{c|}{.539} & .535 & \multicolumn{1}{c|}{77}& \multicolumn{1}{c|}{70}& \multicolumn{1}{c|}{76}& \multicolumn{1}{c|}{77} & 59 \\ \hline
500& \multicolumn{1}{c|}{.237}& \multicolumn{1}{c|}{.241}& \multicolumn{1}{c|}{.241} & \multicolumn{1}{c|}{.243} & .248 & \multicolumn{1}{c|}{.545}& \multicolumn{1}{c|}{.540}& \multicolumn{1}{c|}{.540} & \multicolumn{1}{c|}{.540} & .536 & \multicolumn{1}{c|}{86}& \multicolumn{1}{c|}{79}& \multicolumn{1}{c|}{91}& \multicolumn{1}{c|}{90} & 60 \\ \hline
$\infty$& \multicolumn{1}{c|}{.237}& \multicolumn{1}{c|}{.239}& \multicolumn{1}{c|}{.240} & \multicolumn{1}{c|}{.240} & .248 & \multicolumn{1}{c|}{.545}& \multicolumn{1}{c|}{.540}& \multicolumn{1}{c|}{.541} & \multicolumn{1}{c|}{.540} & .536 & \multicolumn{1}{c|}{86}& \multicolumn{1}{c|}{77}& \multicolumn{1}{c|}{85}& \multicolumn{1}{c|}{85} & 61 \\ \hline
\end{tabular}
}
\caption{Results with respect of HoF size for \textbf{fashion-mnist} dataset. $\infty$ means that there was no limit on HoF size.}
\label{tab:hof-size-fashion-mnist}
\end{table*}

\textbf{Robustness.}
CoEvoRDT was trained separately for max regret and adversarial accuracy, and compared with $4$ SOTA
methods (discussed in the related work section). The results are shown in Tables~\ref{tab:results-max-regret} and~\ref{tab:results-adversarial-accuracy}, respectively. They were also compared 
with the CART algorithm~\cite{breiman2017classification}, 
a popular
method for creating DTs 
for non-perturbed 
training data (not designed for the RDT scenario).
In both tables, the last column (CoEvoRDT+FPRDT) presents the results of 
adding the FPRDT output DT
to the CoEvoRDT initial population, and running CoEvoRDT afterwards.
On the max regret metric, CoEvoRDT clearly outperforms all other competitors on all datasets. Adding the FPRDT tree only narrowly improves its outcome. 
The results support our claim that 
SOTA methods, which cannot directly minimize max regret, perform significantly worse than CoEvoRDT in terms of this metric.
From an adversarial accuracy perspective, 
for 13 out of 20 datasets, CoEvoRDT yielded the best mean results (5 of them statistically significant). For the 
remaining 7 datasets FPRDT method was superior (with statistical significance in 3 cases). For this metrics, adding FPRDT tree to the initial CoEvoRDT population (CoEvoRDT+FPRDT) led to clear advantage versus baseline CoEvoRDT and FPRDT alone. 

For a more detailed analysis, fashion-mnist dataset is selected as one of the largest. 
Detailed results for all other datasets are presented in SM.

\textbf{Runtime comparison}. In general, CoEvoRDT runtime varies from a few seconds to a couple of minutes for the largest datasets.
The average computation time of a single run of the strongest competitor, FPRDT is 2 to 8 times lower than CoEvoRDT. Thus, the natural question which can arise is what if we run FPRDT multiple times to have an equal computation budget and choose the best result. This approach is addressed in Table~\ref{tab:multiple-runs-fashion-mnist}, which presents computation time and the best results in terms of minimax regret and adversarial accuracy for multiple 
runs of FPRDT, CoEvoRDT, and CoEvoRDT initialized with multiple FPRDT outcomes. More runs can notably improve results for all methods. For max regret, even within 100 repeats, FPRDT was not able to find a solution close to the single CoEvoRDT outcome. For adversarial accuracy multiple FPRDT runs outperformed CoEvoRDT, but for greater numbers of repeats (above 20) both methods seem to converge to similar results. Given some constrained computation budget the best option is to run CoEvoRDT + $N$ FPRDT.

\textbf{HoF size and construction.}
Table~\ref{tab:hof-size-fashion-mnist} presents the results for various HoF sizes.
For each size, 
5 variants of constructing HoF are considered:
adding one \textit{mixed tree} from mixed Nash equilibrium after each generation (which is the baseline used in CoEvoRDT), adding all \textit{single trees} from mixed Nash equilibrium (i.e., all the pure strategies with positive probability), adding only one highest-fitness tree from the population, adding top $K$ highest-fitness individuals from the current population (where $K$ is the number of trees from Nash mixed equilibrium) and adding one mixed tree composed of top $K$ highest-fitness trees with equal probabilities.
Firstly, it is clear that even small HoF significantly improves results for all variants. Moreover, adding a Nash mixed tree seems to be the best option, while adding only one
highest-fitness 
tree is the worst approach. 
The advantage of Nash mixed tree and top $K$ as a mixed tree with equal probabilities 
over Nash single trees and top $K$ trees shows the benefit of 
using mixed trees in the HoF. It may 
stem from the fact that mixed tree is 
\textit{more robust} to various perturbations, and 
consequently the perturbation population is forced to find better perturbations to \textit{outplay} those mixed trees. As a result, 
the DT population is 
forced to create even more robust trees.
At the same time, the straightforward approach of creating a mixed tree of highest-fitness individuals is less powerful than a mixed tree from mixed Nash equilibrium.

The \textbf{generation limit} of 
CoEvoRDT 
was set to 1000, but in practice, it was rarely reached and the other stop condition (no improvement of best-found solution) was fulfilled first. The average number of generations across all datasets was 385. The lowest average number was observed for the diabetes 
(152), and the 
highest for cifar10:0v5 (865). 
The depth of DTs generated by CoEvoRDT
varies from 6 to 23 which is not much different than 
DTs created by other methods.

CoEvoRDT \textbf{memory consumption} is low and does not exceed 150 MB for the largest datasets.

\section{Conclusions}
In this paper, we present CoEvoRDT, a novel coevolutionary algorithm designed to construct robust decision trees capable of handling 
perturbed high-dimensional data.
Our motivation stems from the vulnerability of traditional 
DT algorithms to adversarial perturbations and the limitations of existing defensive algorithms in optimizing specific metrics like max regret. The 
flexibility of CoEvoRDT in accommodating various target metrics makes it adaptable to a wide range of applications and 
domains, 
including when robustness is mixed with other objectives such as fairness. 
We propose a 
novel game-theoretic approach 
to constructing the Hall of Fame with Mixed Nash Equilibrium, which significantly contributes to 
the DTs robustness and convergence speed.
CoEvoRDT can additionally integrate results from another strong and fast method into the initial population, if one is available, to further improve performance.

CoEvoRDT was comprehensively tested on 20 popular benchmark datasets and compared with 
4 SOTA algorithms, 
presenting on par 
performance to the best competitive method in adversarial accuracy metrics, and outperforming all 
competitors in terms of minimax regret.


Our future work focuses on investigating the potential of 
implementing CoEvoRDT as a multi-population algorithm, such as the island model~\cite{skolicki2005analysis}, to speed up convergence and potentially further 
boost its performance.


\section{Acknowledgements}
We gratefully acknowledge the funding support by program “Excellence initiative—research university” for the AGH University of Krakow as well as the ARTIQ project: UMO-2021/01/2/ST6/00004 and ARTIQ/0004/2021

\bibliography{aaai24}

\clearpage


\section*{Supplementary material}

# Coevolutionary Algorithm for Building Robust Decision Trees under Minimax Regret
# – Supplementary material –

## Computation times

Table 1 presents a computation times comparison for CoEvoRDT and state-of-the-art methods. Clearly, the fastest (but with the weakest results) is CART which does not consider robustness. Evolutionary approaches (Meta Silvae and CoEvoRDT) are a few times slower than other methods which bases on splitting techniques more appropriate for creating robust decision trees. Datasets with the longest computation time are from *cifar* family because of the greatest number of features (3072) which makes for lots of potential splitting options as well as greatly increases possible perturbations space.

## CoEvoRDT parameterization

CoEvoRDT parameterization process was performed on cod-rna dataset (Uzilov, Keegan, and Mathews 2006) with 9 features, 2 classes, and 48565 instances. This dataset was not used in CoEvoRDT experimental evaluation described in the main paper. The algorithm was executed 10000 times with parameter values set randomly, i.e. for each run each parameter was drawn uniformly from some predefined set of values:

- decision trees population size
  $N_T : \{10, 20, 50, 100, 200, 500, \mathbf{1000}\}$
- perturbations population size
  $N_P : \{100, 200, 500, 1000, 2000, 5000, \mathbf{10000}\}$
- number of consecutive generations for each population
  $l_c : \{1, 2, 5, 10, \mathbf{20}, 50, 100\}$
- the number of the best individuals from the decision trees population involved in the perturbations evaluation
  $N_{\text{top}} : \{1, 2, 5, 10, \mathbf{20}, 50, 100, 200\}$
- crossover probability
  $p_c : \{0.0, 0.1, 0.2, 0.3, 0.4, 0.5, 0.6, 0.7, \mathbf{0.8}, 0.9, 1.0\}$
- mutation probability
  $p_m : \{0.0, 0.1, 0.2, 0.3, 0.4, \mathbf{0.5}, 0.6, 0.7, 0.8, 0.9, 1.0\}$
- selection pressure
  $p_s : \{0.5, 0.6, 0.7, 0.8, \mathbf{0.9}, 1.0\}$
- HoF size
  $N_{\text{HoF}} : \{0, 10, 20, 50, 100, 200, \mathbf{500}\}$
- generations without improvement limit
  $l_c : \{5, 10, 20, \mathbf{50}, 100, 200\}$
- generations limit
  $l_g : \{100, 200, 500, \mathbf{1000, 2000, 5000}\}$

Best values (with the lowest average minimax regret across all runs) are bolded.

The following CoEvoRDT parameter values were set in the experiments and comparison with state-of-the-art methods: decision tree population size $N_T = 200$, perturbation population size $N_P = 500$, number of consecutive generations for each population $l_c = 20$, number of best individuals from the DT population involved in the perturbations evaluation $N_{\text{top}} = 20$, crossover probability $p_c = 0.8$, mutation probability $p_m = 0.5$, selection pressure $p_s = 0.9$, elite size $e = 2$, HoF size $N_{\text{HoF}} = 200$, generations without improvement limit $l_c = 50$, generations limit $l_g = 1000$,

Below we present a more detailed analysis for 4 parameters which appeared to be the most interesting.

Figure 2 shows the results (average minimax regret and computation time) according to the number of individuals in the decision trees population. Clearly, the bigger the population size the better the results since more potential solutions were checked. However, the improvement for large populations (500 and 1000 individuals) is insignificant compared to additional computation time which needs to be dedicated. Thus, the best trade-off between results quality and computation time seems to be $N_T = 200$ and this value was adopted as a recommended value.

A similar relationship can be observed for perturbations population size in Figure 3. More perturbations mean more accurate assessment for decision trees from the adversarial population but it also costs more computational power. As in the previous case $N_P = 500$ was chosen as a good compromise between minimax regret and computation time.

The results of tuning $l_c$ parameter are presented in Figure 4. Small values ($l_c \leq 5$ - frequent switching between populations), as well as big ones ($l_c \geq 50$) result in performance deterioration. Infrequent switching makes one population dominant and the other one stagnates over a long time with no chances to respond to the evolved individuals from the other population. On the other hand, in frequent switching the population may not have enough generations to find adequate response to adversarial individuals. At the same time, for all tested values the computation time is similar. Hence, $l_c = 20$ was adopted as a recommended value.

The next tuned parameter was $N_{\text{top}}$, i.e. the number of the

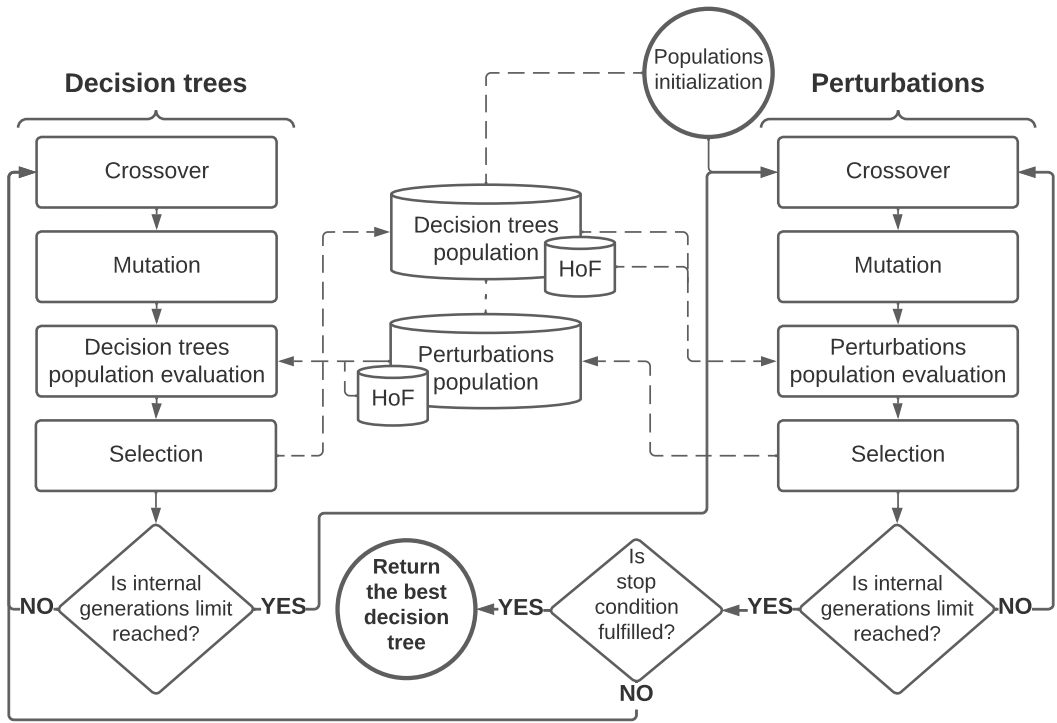

Figure 1: A high-level overview of the CoEvoDT algorithm.

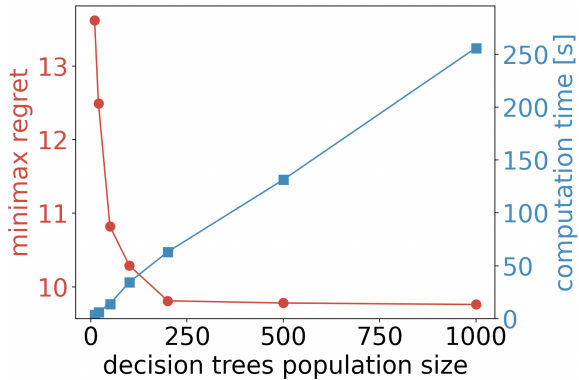

Figure 2: Comparison of minimax regret and computation size for **decision trees population size** values.

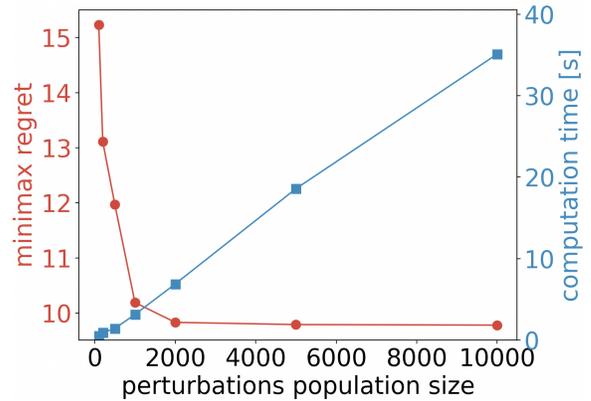

Figure 3: Comparison of minimax regret and computation size for **perturbations population size** values.

best individuals from the decision trees population involved in the perturbations evaluation. The results are presented in Figure 5 confirm the conjecture formulated in the main paper about the harmfulness of using the whole decision trees population ($N_{\text{top}} = 200$). Also, small values of this parameter ($N_{\text{top}} < 5$) lead to weaker results, A deeper examination of population structure in such cases reviles the presence of oscillations. In the extreme case of $N_{\text{top}} = 1$ (evaluation of a given perturbation is based on the best decision tree only) we observed that the perturbations population quickly losses diversity. Individuals in the population become similar to one another because they are optimized with respect to only one decision tree. As a result, the perturbations population returns a good response only to this particular decision tree, and in the next coevolution phase, the decision trees population is able to find with ease another solution for which there is no good response in the perturbations population. Afterwards, the whole perturbations population again adapts to the new best decision tree and "forgets" the previous ones. $N_{\text{top}} = 20$ appeared to be the best compromise between these two extremes (Figure 5).

| dataset | CART | Meta Silvae | RIGDT-h | GROOT | FPRDT | CoEvoRDT | CoEvoRDT+FPRDT |
|---|---|---|---|---|---|---|---|
| ionos | 0 | 2 | 1 | 1 | 1 | 2 | 3 |
| breast | 0 | 2 | 1 | 1 | 1 | 2 | 3 |
| diabetes | 0 | 2 | 1 | 1 | 1 | 3 | 4 |
| bank | 1 | 5 | 2 | 2 | 2 | 6 | 8 |
| Japan3v4 | 1 | 8 | 3 | 3 | 3 | 9 | 12 |
| spam | 1 | 11 | 4 | 4 | 4 | 13 | 17 |
| GesDvP | 1 | 10 | 4 | 4 | 4 | 11 | 15 |
| har1v2 | 1 | 11 | 4 | 4 | 4 | 12 | 15 |
| wine | 2 | 6 | 6 | 6 | 6 | 2 | 8 |
| collision-det | 9 | 26 | 16 | 18 | 16 | 17 | 35 |
| mnist-1-5 | 8 | 14 | 8 | 8 | 8 | 13 | 20 |
| mnist-2-6 | 6 | 19 | 8 | 8 | 7 | 24 | 31 |
| mnist | 17 | 62 | 22 | 21 | 21 | 68 | 93 |
| f-mnist2v5 | 7 | 20 | 9 | 9 | 8 | 23 | 34 |
| f-mnist3v4 | 7 | 23 | 10 | 9 | 9 | 25 | 36 |
| f-mnist7v9 | 7 | 21 | 10 | 9 | 9 | 26 | 35 |
| f-mnist | 18 | 60 | 21 | 19 | 19 | 79 | 97 |
| cifar10:0v5 | 21 | 112 | 42 | 39 | 40 | 146 | 191 |
| cifar10:0v6 | 22 | 107 | 44 | 44 | 42 | 126 | 174 |
| cifar10:4v8 | 21 | 106 | 41 | 41 | 41 | 111 | 163 |

Table 1: Comparison of methods' computation times (in seconds).

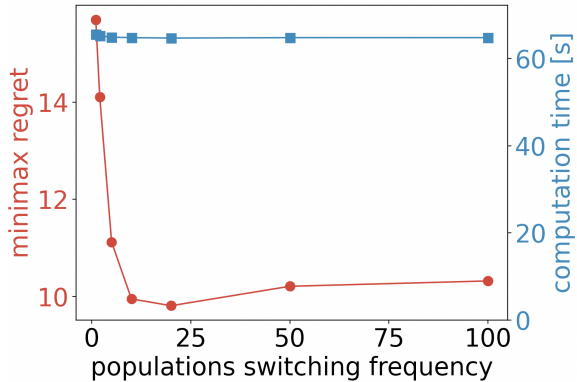

Figure 4: Comparison of minimax regret and computation size for **populations evaluation switching frequency** $l_c$).

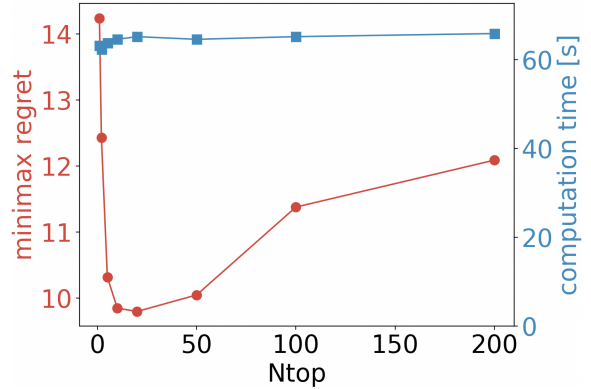

Figure 5: Comparison of minimax regret and computation size for **the number of best individuals used to evaluate the perturbations** ($N_{\text{top}}$).

## Metrics calculation

Calculating the exact value of adversarial accuracy or minimax regret is not straightforward. It requires finding a perturbation that minimizes accuracy or maximizes regret from the infinite set of possible perturbations. Since this task is not trivial, we decided to estimate the real values of these metrics by drawing a uniformly random subset $P$ of possible perturbations and then calculating the performance of all models on this subset.

In order to assess how large this subset should be to fairly estimate the performance of models, we chose 5 datasets with different $\varepsilon$ values and ran each tested method on each dataset 5 times. This resulted in 25 decision trees. We then checked the following values for the size of $P$: $10^2, 10^3, 10^4, 10^5, 10^6, 10^7$. For each value of $P$, we drew a given number of random perturbations and then evaluated all 25 decision trees using minimax regret and adversarial accuracy. The results for all models were then averaged. This procedure was repeated 20 times (each time a new subset of perturbations was drawn, but the 25 models remained the same) for each value of P. The mean value and standard error for the tested values of $P$ are presented in Table 2. It shows that the standard error value decreases with the size of $P$. This is expected, as a larger subset of perturbations allows for a more thorough search of the space of possible perturbations and indicates that the results are becoming more reliable.

The standard error for small values ($\log_{10} |P| \leq 4$) is high, which shows that the calculated metrics values are un-

reliable. However, for $P$ sizes of at least $10^5$, the difference between multiple perturbations drawn is small and the results are stabilized. This does not indicate how close to the exact (real) values we are, but it does show that $10^5$ is a large enough size of drawn perturbations sample to fairly assess and compare tested models.

|  | minimax regret | | adversarial accuracy | |
| --- | --- | --- | --- | --- |
| $log_{10}|P|$ | average | std error | average | std error |
| 2 | 0.0964 | 0.0065 | 0.740 | 0.0088 |
| 3 | 0.0968 | 0.0053 | 0.737 | 0.0056 |
| 4 | 0.0972 | 0.0029 | 0.733 | 0.0036 |
| 5 | 0.0976 | 0.0005 | 0.729 | 0.0006 |
| 6 | 0.0977 | 0.0002 | 0.728 | 0.0005 |
| 7 | 0.0977 | 0.0001 | 0.728 | 0.0003 |

Table 2: Mean value and standard error of adversarial accuracy and minimax regret for different values of the size of random perturbations sample used to their calculation.

The role of Table 2 was to determine what perturbation set size should be used in experimental procedure from the main paper (Tables 2 and 3). Table 2 presented above explores how varying perturbation set size impacts results (variance arises from perturbation-related randomness), while the main paper's Tables 2 and 3 employ a different setup: for each method and each individual run, a new decision tree was generated, and subsequently, these generated decision trees were evaluated with a uniform and unchanging set of perturbations (with size of $10^5$). In this context, the variance arises from the diversity among the decision tree models themselves.

## Results for all datasets

Tables 3–21 show the computation time and the best outcomes in terms of minimax regret and adversarial accuracy resulting from multiple runs of FPRDT, CoEvoRDT, and CoEvoRDT initialized with various FPRDT outcomes. Obtained results confirm statements from the main paper that additional runs can significantly enhance the results for all methods. With regards to minimax regret FPRDT failed to approach a solution comparable to the individual outcome of CoEvoRDT. In terms of adversarial accuracy, multiple runs of FPRDT outperformed CoEvoRDT, yet as the number of iterations increased, both approaches appeared to converge towards similar outcomes.

Tables 22–40 display the outcomes across various sizes of the Hall of Fame (HoF). For each size, five different approaches for constructing the HoF are evaluated: including one *mixed tree* from the mixed Nash equilibrium after each generation (which serves as the baseline in CoEvoRDT), incorporating all *single trees* from the mixed Nash equilibrium, integrating only the single highest-fitness tree from the population, appending the top $K$ highest-fitness individuals from the current population (with $K$ representing the number of trees from the Nash mixed equilibrium), and adding one mixed tree composed of the top $K$ highest-fitness trees with equal probabilities. Even a small HoF size significantly enhances results for all the tested approaches. Furthermore, the strategy of introducing a Nash mixed tree appears to yield the most favorable outcomes, whereas relying solely on the addition of the single highest-fitness tree yields the least favorable results.

|   | minimax regret | | | | adversarial accuracy | | | | computation time [s] | | | |
|---|---|---|---|---|---|---|---|---|---|---|---|---|
| N | N FPRDT | N CoEvoRDT | CoEvoRDT + N FPRDT | N CoEvoRDT + N FPRDT | N FPRDT | N CoEvoRDT | CoEvoRDT + N FPRDT | N CoEvoRDT + N FPRDT | N FPRDT | N CoEvoRDT | CoEvoRDT + N FPRDT | N CoEvoRDT + N FPRDT |
| 1 | .061 | .052 | .052 | .052 | .795 | .791 | .795 | .795 | 1 | 2 | 3 | 3 |
| 2 | .060 | .051 | .052 | .051 | .800 | .797 | .802 | .803 | 2 | 4 | 4 | 5 |
| 3 | .060 | .051 | .051 | .051 | .802 | .805 | .803 | .811 | 3 | 6 | 5 | 10 |
| 4 | .060 | .051 | .051 | .051 | .810 | .812 | .814 | .818 | 4 | 9 | 8 | 12 |
| 5 | .059 | .051 | .051 | .051 | .821 | .827 | .823 | .828 | 5 | 11 | 11 | 16 |
| 10 | .058 | .050 | .050 | .050 | .842 | .836 | .835 | .840 | 10 | 19 | 12 | 28 |
| 20 | .056 | .048 | .049 | .048 | .850 | .849 | .848 | .850 | 20 | 37 | 22 | 57 |
| 50 | .055 | .048 | .048 | .048 | .851 | .853 | .851 | .852 | 49 | 104 | 40 | 164 |
| 100 | .055 | .048 | .049 | .048 | .853 | .858 | .852 | .857 | 99 | 192 | 109 | 275 |

Table 3: Best results of repeated $N$ algorithms' runs for **ionos** dataset. In CoEvoRDT + N FPRDT resulting DTs from N independent runs of FPRDT were incorporated into CoEvoRDT initial population.

|   | minimax regret | | | | adversarial accuracy | | | | computation time [s] | | | |
|---|---|---|---|---|---|---|---|---|---|---|---|---|
| N | N FPRDT | N CoEvoRDT | CoEvoRDT + N FPRDT | N CoEvoRDT + N FPRDT | N FPRDT | N CoEvoRDT | CoEvoRDT + N FPRDT | N CoEvoRDT + N FPRDT | N FPRDT | N CoEvoRDT | CoEvoRDT + N FPRDT | N CoEvoRDT + N FPRDT |
| 1 | .057 | .049 | .049 | .049 | .876 | .885 | .889 | .889 | 1 | 2 | 3 | 3 |
| 2 | .056 | .049 | .049 | .049 | .881 | .893 | .892 | .896 | 2 | 4 | 4 | 6 |
| 3 | .056 | .049 | .048 | .048 | .887 | .896 | .902 | .911 | 3 | 6 | 6 | 10 |
| 4 | .056 | .048 | .048 | .048 | .892 | .917 | .908 | .916 | 4 | 8 | 8 | 12 |
| 5 | .056 | .048 | .048 | .048 | .909 | .922 | .924 | .927 | 5 | 10 | 11 | 15 |
| 10 | .054 | .048 | .047 | .047 | .921 | .936 | .936 | .940 | 10 | 18 | 12 | 28 |
| 20 | .052 | .046 | .046 | .046 | .935 | .946 | .946 | .949 | 21 | 37 | 19 | 61 |
| 50 | .051 | .046 | .046 | .045 | .940 | .955 | .953 | .949 | 51 | 102 | 39 | 161 |
| 100 | .051 | .045 | .046 | .045 | .938 | .954 | .947 | .957 | 95 | 210 | 103 | 297 |

Table 4: Best results of repeated $N$ algorithms' runs for **breast** dataset. In CoEvoRDT + N FPRDT resulting DTs from N independent runs of FPRDT were incorporated into CoEvoRDT initial population.

|   | minimax regret | | | | adversarial accuracy | | | | computation time [s] | | | |
|---|---|---|---|---|---|---|---|---|---|---|---|---|
| N | N FPRDT | N CoEvoRDT | CoEvoRDT + N FPRDT | N CoEvoRDT + N FPRDT | N FPRDT | N CoEvoRDT | CoEvoRDT + N FPRDT | N CoEvoRDT + N FPRDT | N FPRDT | N CoEvoRDT | CoEvoRDT + N FPRDT | N CoEvoRDT + N FPRDT |
| 1 | .117 | .096 | .094 | .094 | .648 | .617 | .648 | .648 | 1 | 3 | 4 | 4 |
| 2 | .115 | .096 | .093 | .093 | .653 | .622 | .656 | .657 | 2 | 6 | 5 | 9 |
| 3 | .115 | .095 | .093 | .093 | .657 | .626 | .654 | .663 | 3 | 9 | 7 | 13 |
| 4 | .114 | .095 | .093 | .093 | .662 | .633 | .660 | .663 | 4 | 12 | 12 | 16 |
| 5 | .114 | .094 | .093 | .092 | .670 | .642 | .673 | .673 | 5 | 16 | 15 | 21 |
| 10 | .111 | .092 | .091 | .090 | .687 | .655 | .684 | .684 | 11 | 30 | 17 | 39 |
| 20 | .107 | .090 | .088 | .088 | .690 | .661 | .692 | .691 | 21 | 64 | 29 | 88 |
| 50 | .105 | .089 | .088 | .087 | .695 | .667 | .694 | .694 | 50 | 158 | 62 | 202 |
| 100 | .105 | .089 | .088 | .087 | .694 | .667 | .695 | .692 | 105 | 314 | 154 | 426 |

Table 5: Best results of repeated $N$ algorithms' runs for **diabetes** dataset. In CoEvoRDT + N FPRDT resulting DTs from N independent runs of FPRDT were incorporated into CoEvoRDT initial population.

|   | minimax regret | | | | adversarial accuracy | | | | computation time [s] | | | |
|---|---|---|---|---|---|---|---|---|---|---|---|---|
| N | N FPRDT | N CoEvoRDT | CoEvoRDT + N FPRDT | N CoEvoRDT + N FPRDT | N FPRDT | N CoEvoRDT | CoEvoRDT + N FPRDT | N CoEvoRDT + N FPRDT | N FPRDT | N CoEvoRDT | CoEvoRDT + N FPRDT | N CoEvoRDT + N FPRDT |
| 1 | .089 | .076 | .076 | .076 | .658 | .657 | .663 | .663 | 2 | 6 | 8 | 8 |
| 2 | .088 | .076 | .075 | .075 | .659 | .663 | .669 | .668 | 4 | 12 | 11 | 15 |
| 3 | .088 | .075 | .076 | .075 | .664 | .669 | .671 | .678 | 6 | 17 | 14 | 22 |
| 4 | .087 | .075 | .075 | .075 | .673 | .674 | .677 | .683 | 9 | 22 | 21 | 33 |
| 5 | .087 | .075 | .075 | .075 | .679 | .686 | .685 | .687 | 10 | 28 | 29 | 40 |
| 10 | .085 | .074 | .073 | .073 | .697 | .694 | .694 | .698 | 21 | 62 | 32 | 83 |
| 20 | .080 | .071 | .071 | .071 | .703 | .705 | .707 | .706 | 42 | 115 | 66 | 154 |
| 50 | .080 | .071 | .071 | .071 | .709 | .705 | .709 | .711 | 101 | 318 | 131 | 450 |
| 100 | .080 | .070 | .071 | .070 | .707 | .713 | .711 | .709 | 206 | 650 | 291 | 922 |

Table 6: Best results of repeated $N$ algorithms' runs for **bank** dataset. In CoEvoRDT + N FPRDT resulting DTs from N independent runs of FPRDT were incorporated into CoEvoRDT initial population.

| | minimax regret | | | | adversarial accuracy | | | | computation time [s] | | | |
|---|---|---|---|---|---|---|---|---|---|---|---|---|
| N | N FPRDT | N CoEvoRDT | CoEvoRDT + N FPRDT | N CoEvoRDT + N FPRDT | N FPRDT | N CoEvoRDT | CoEvoRDT + N FPRDT | N CoEvoRDT + N FPRDT | N FPRDT | N CoEvoRDT | CoEvoRDT + N FPRDT | N CoEvoRDT + N FPRDT |
| 1 | .066 | .062 | .061 | .061 | .667 | .665 | .668 | .668 | 3 | 9 | 12 | 12 |
| 2 | .066 | .062 | .061 | .061 | .671 | .668 | .673 | .675 | 6 | 17 | 15 | 24 |
| 3 | .065 | .061 | .060 | .060 | .673 | .676 | .678 | .678 | 8 | 26 | 24 | 32 |
| 4 | .064 | .061 | .060 | .060 | .685 | .682 | .685 | .683 | 12 | 36 | 33 | 51 |
| 5 | .064 | .061 | .060 | .060 | .692 | .692 | .691 | .695 | 16 | 45 | 45 | 57 |
| 10 | .063 | .060 | .059 | .059 | .706 | .703 | .700 | .703 | 30 | 90 | 51 | 108 |
| 20 | .060 | .058 | .057 | .057 | .710 | .710 | .712 | .714 | 55 | 188 | 92 | 261 |
| 50 | .059 | .057 | .057 | .057 | .719 | .715 | .713 | .713 | 147 | 420 | 184 | 596 |
| 100 | .059 | .057 | .057 | .056 | .716 | .716 | .715 | .715 | 277 | 976 | 438 | 1256 |

Table 7: Best results of repeated $N$ algorithms' runs for **Japan3v4** dataset. In CoEvoRDT + N FPRDT resulting DTs from N independent runs of FPRDT were incorporated into CoEvoRDT initial population.

| | minimax regret | | | | adversarial accuracy | | | | computation time [s] | | | |
|---|---|---|---|---|---|---|---|---|---|---|---|---|
| N | N FPRDT | N CoEvoRDT | CoEvoRDT + N FPRDT | N CoEvoRDT + N FPRDT | N FPRDT | N CoEvoRDT | CoEvoRDT + N FPRDT | N CoEvoRDT + N FPRDT | N FPRDT | N CoEvoRDT | CoEvoRDT + N FPRDT | N CoEvoRDT + N FPRDT |
| 1 | .074 | .070 | .069 | .069 | .746 | .751 | .753 | .753 | 4 | 13 | 17 | 17 |
| 2 | .073 | .070 | .069 | .068 | .748 | .760 | .756 | .757 | 8 | 28 | 20 | 37 |
| 3 | .073 | .069 | .069 | .068 | .752 | .760 | .764 | .765 | 12 | 38 | 36 | 54 |
| 4 | .072 | .069 | .068 | .068 | .762 | .775 | .768 | .777 | 16 | 55 | 47 | 66 |
| 5 | .072 | .069 | .068 | .068 | .773 | .788 | .777 | .781 | 18 | 62 | 69 | 78 |
| 10 | .070 | .068 | .067 | .066 | .783 | .795 | .793 | .796 | 40 | 127 | 71 | 183 |
| 20 | .067 | .066 | .065 | .064 | .794 | .806 | .800 | .801 | 76 | 255 | 134 | 331 |
| 50 | .066 | .065 | .065 | .064 | .801 | .807 | .806 | .803 | 218 | 627 | 261 | 795 |
| 100 | .066 | .065 | .064 | .064 | .803 | .815 | .808 | .811 | 394 | 1193 | 581 | 1722 |

Table 8: Best results of repeated $N$ algorithms' runs for **spam** dataset. In CoEvoRDT + N FPRDT resulting DTs from N independent runs of FPRDT were incorporated into CoEvoRDT initial population.

| | minimax regret | | | | adversarial accuracy | | | | computation time [s] | | | |
|---|---|---|---|---|---|---|---|---|---|---|---|---|
| N | N FPRDT | N CoEvoRDT | CoEvoRDT + N FPRDT | N CoEvoRDT + N FPRDT | N FPRDT | N CoEvoRDT | CoEvoRDT + N FPRDT | N CoEvoRDT + N FPRDT | N FPRDT | N CoEvoRDT | CoEvoRDT + N FPRDT | N CoEvoRDT + N FPRDT |
| 1 | .131 | .114 | .114 | .114 | .735 | .740 | .741 | .741 | 4 | 11 | 15 | 15 |
| 2 | .130 | .114 | .113 | .114 | .737 | .746 | .745 | .748 | 8 | 21 | 19 | 29 |
| 3 | .129 | .113 | .113 | .113 | .740 | .751 | .749 | .759 | 12 | 34 | 29 | 42 |
| 4 | .128 | .112 | .112 | .112 | .754 | .766 | .757 | .759 | 16 | 46 | 43 | 55 |
| 5 | .127 | .112 | .112 | .112 | .761 | .771 | .770 | .768 | 19 | 57 | 55 | 71 |
| 10 | .124 | .110 | .111 | .110 | .772 | .782 | .781 | .785 | 43 | 115 | 71 | 151 |
| 20 | .119 | .107 | .107 | .106 | .784 | .793 | .793 | .790 | 80 | 198 | 134 | 267 |
| 50 | .117 | .106 | .106 | .106 | .791 | .795 | .788 | .791 | 205 | 540 | 194 | 748 |
| 100 | .117 | .106 | .106 | .106 | .788 | .801 | .796 | .798 | 399 | 1159 | 535 | 1441 |

Table 9: Best results of repeated $N$ algorithms' runs for **GesDvP** dataset. In CoEvoRDT + N FPRDT resulting DTs from N independent runs of FPRDT were incorporated into CoEvoRDT initial population.

| | minimax regret | | | | adversarial accuracy | | | | computation time [s] | | | |
|---|---|---|---|---|---|---|---|---|---|---|---|---|
| N | N FPRDT | N CoEvoRDT | CoEvoRDT + N FPRDT | N CoEvoRDT + N FPRDT | N FPRDT | N CoEvoRDT | CoEvoRDT + N FPRDT | N CoEvoRDT + N FPRDT | N FPRDT | N CoEvoRDT | CoEvoRDT + N FPRDT | N CoEvoRDT + N FPRDT |
| 1 | .068 | .064 | .064 | .064 | .804 | .818 | .820 | .820 | 4 | 12 | 15 | 15 |
| 2 | .067 | .064 | .064 | .063 | .808 | .828 | .823 | .828 | 8 | 22 | 21 | 31 |
| 3 | .067 | .063 | .063 | .063 | .817 | .832 | .828 | .839 | 12 | 36 | 28 | 53 |
| 4 | .067 | .063 | .063 | .063 | .824 | .839 | .840 | .844 | 16 | 52 | 42 | 74 |
| 5 | .066 | .063 | .063 | .063 | .835 | .853 | .847 | .852 | 19 | 55 | 63 | 73 |
| 10 | .065 | .062 | .062 | .062 | .848 | .868 | .866 | .866 | 38 | 122 | 65 | 153 |
| 20 | .062 | .060 | .060 | .060 | .858 | .875 | .876 | .878 | 84 | 239 | 136 | 317 |
| 50 | .061 | .059 | .059 | .059 | .868 | .877 | .880 | .880 | 200 | 608 | 272 | 852 |
| 100 | .061 | .060 | .059 | .059 | .869 | .887 | .882 | .881 | 403 | 1253 | 656 | 1493 |

Table 10: Best results of repeated $N$ algorithms' runs for **har1v2** dataset. In CoEvoRDT + N FPRDT resulting DTs from N independent runs of FPRDT were incorporated into CoEvoRDT initial population.

|   | minimax regret | | | | adversarial accuracy | | | | computation time [s] | | | |
|---|---|---|---|---|---|---|---|---|---|---|---|---|
| N | N FPRDT | N CoEvoRDT | CoEvoRDT + N FPRDT | N CoEvoRDT + N FPRDT | N FPRDT | N CoEvoRDT | CoEvoRDT + N FPRDT | N CoEvoRDT + N FPRDT | N FPRDT | N CoEvoRDT | CoEvoRDT + N FPRDT | N CoEvoRDT + N FPRDT |
| 1 | .109 | .090 | .090 | .090 | .674 | .688 | .692 | .692 | 6 | 2 | 8 | 8 |
| 2 | .108 | .089 | .089 | .089 | .675 | .692 | .696 | .699 | 11 | 4 | 13 | 15 |
| 3 | .107 | .089 | .089 | .089 | .682 | .699 | .702 | .707 | 19 | 6 | 16 | 23 |
| 4 | .107 | .089 | .088 | .088 | .686 | .707 | .707 | .709 | 26 | 8 | 15 | 33 |
| 5 | .107 | .088 | .089 | .089 | .699 | .715 | .715 | .721 | 30 | 11 | 19 | 38 |
| 10 | .103 | .087 | .087 | .087 | .713 | .725 | .731 | .728 | 63 | 21 | 21 | 86 |
| 20 | .099 | .084 | .084 | .084 | .724 | .736 | .739 | .743 | 119 | 36 | 32 | 148 |
| 50 | .098 | .084 | .084 | .084 | .722 | .740 | .741 | .738 | 295 | 98 | 46 | 433 |
| 100 | .098 | .083 | .084 | .084 | .723 | .741 | .737 | .745 | 599 | 197 | 107 | 716 |

Table 11: Best results of repeated $N$ algorithms' runs for **wine** dataset. In CoEvoRDT + N FPRDT resulting DTs from N independent runs of FPRDT were incorporated into CoEvoRDT initial population.

|   | minimax regret | | | | adversarial accuracy | | | | computation time [s] | | | |
|---|---|---|---|---|---|---|---|---|---|---|---|---|
| N | N FPRDT | N CoEvoRDT | CoEvoRDT + N FPRDT | N CoEvoRDT + N FPRDT | N FPRDT | N CoEvoRDT | CoEvoRDT + N FPRDT | N CoEvoRDT + N FPRDT | N FPRDT | N CoEvoRDT | CoEvoRDT + N FPRDT | N CoEvoRDT + N FPRDT |
| 1 | .091 | .061 | .059 | .059 | .792 | .798 | .803 | .803 | 16 | 17 | 35 | 35 |
| 2 | .091 | .061 | .059 | .058 | .793 | .801 | .813 | .809 | 33 | 31 | 46 | 69 |
| 3 | .090 | .061 | .058 | .059 | .801 | .806 | .818 | .816 | 47 | 46 | 65 | 96 |
| 4 | .090 | .060 | .058 | .058 | .810 | .821 | .825 | .823 | 63 | 74 | 80 | 141 |
| 5 | .089 | .060 | .058 | .058 | .817 | .834 | .829 | .831 | 79 | 89 | 105 | 165 |
| 10 | .087 | .059 | .057 | .057 | .832 | .845 | .847 | .849 | 151 | 171 | 110 | 299 |
| 20 | .082 | .057 | .055 | .055 | .845 | .852 | .852 | .862 | 304 | 360 | 208 | 619 |
| 50 | .082 | .057 | .055 | .055 | .852 | .862 | .855 | .864 | 847 | 919 | 396 | 1637 |
| 100 | .081 | .057 | .055 | .055 | .852 | .867 | .861 | .861 | 1643 | 1685 | 897 | 3335 |

Table 12: Best results of repeated $N$ algorithms' runs for **collision-det** dataset. In CoEvoRDT + N FPRDT resulting DTs from N independent runs of FPRDT were incorporated into CoEvoRDT initial population.

|   | minimax regret | | | | adversarial accuracy | | | | computation time [s] | | | |
|---|---|---|---|---|---|---|---|---|---|---|---|---|
| N | N FPRDT | N CoEvoRDT | CoEvoRDT + N FPRDT | N CoEvoRDT + N FPRDT | N FPRDT | N CoEvoRDT | CoEvoRDT + N FPRDT | N CoEvoRDT + N FPRDT | N FPRDT | N CoEvoRDT | CoEvoRDT + N FPRDT | N CoEvoRDT + N FPRDT |
| 1 | .067 | .055 | .055 | .055 | .966 | .964 | .969 | .969 | 8 | 13 | 20 | 20 |
| 2 | .066 | .054 | .055 | .055 | .968 | .966 | .971 | .971 | 15 | 26 | 30 | 40 |
| 3 | .066 | .054 | .055 | .055 | .968 | .968 | .974 | .975 | 23 | 41 | 44 | 60 |
| 4 | .065 | .054 | .054 | .054 | .972 | .971 | .975 | .976 | 33 | 47 | 52 | 85 |
| 5 | .065 | .054 | .054 | .054 | .974 | .975 | .977 | .977 | 43 | 63 | 61 | 113 |
| 10 | .064 | .053 | .053 | .053 | .980 | .978 | .982 | .983 | 72 | 130 | 74 | 190 |
| 20 | .061 | .052 | .052 | .051 | .984 | .980 | .986 | .987 | 161 | 237 | 131 | 418 |
| 50 | .060 | .051 | .052 | .051 | .984 | .982 | .986 | .985 | 395 | 605 | 248 | 1041 |
| 100 | .060 | .051 | .052 | .051 | .984 | .984 | .985 | .987 | 770 | 1272 | 609 | 2095 |

Table 13: Best results of repeated $N$ algorithms' runs for **mnist:1v5** dataset. In CoEvoRDT + N FPRDT resulting DTs from N independent runs of FPRDT were incorporated into CoEvoRDT initial population.

|   | minimax regret | | | | adversarial accuracy | | | | computation time [s] | | | |
|---|---|---|---|---|---|---|---|---|---|---|---|---|
| N | N FPRDT | N CoEvoRDT | CoEvoRDT + N FPRDT | N CoEvoRDT + N FPRDT | N FPRDT | N CoEvoRDT | CoEvoRDT + N FPRDT | N CoEvoRDT + N FPRDT | N FPRDT | N CoEvoRDT | CoEvoRDT + N FPRDT | N CoEvoRDT + N FPRDT |
| 1 | .069 | .055 | .054 | .054 | .922 | .917 | .922 | .922 | 7 | 24 | 31 | 31 |
| 2 | .068 | .055 | .054 | .054 | .927 | .919 | .928 | .929 | 15 | 43 | 40 | 62 |
| 3 | .068 | .054 | .053 | .053 | .932 | .925 | .932 | .934 | 19 | 70 | 44 | 102 |
| 4 | .068 | .054 | .053 | .053 | .942 | .939 | .943 | .944 | 30 | 91 | 59 | 115 |
| 5 | .067 | .054 | .053 | .053 | .953 | .947 | .954 | .956 | 37 | 123 | 69 | 141 |
| 10 | .066 | .053 | .052 | .052 | .961 | .954 | .963 | .963 | 68 | 219 | 98 | 280 |
| 20 | .063 | .051 | .051 | .050 | .965 | .958 | .965 | .967 | 129 | 486 | 145 | 618 |
| 50 | .062 | .051 | .050 | .050 | .969 | .963 | .969 | .970 | 349 | 1311 | 397 | 1669 |
| 100 | .062 | .051 | .050 | .050 | .971 | .969 | .972 | .973 | 680 | 2385 | 675 | 3281 |

Table 14: Best results of repeated $N$ algorithms' runs for **mnist:2v6** dataset. In CoEvoRDT + N FPRDT resulting DTs from N independent runs of FPRDT were incorporated into CoEvoRDT initial population.

|   | minimax regret | | | | adversarial accuracy | | | | computation time [s] | | | |
|---|---|---|---|---|---|---|---|---|---|---|---|---|
| N | N FPRDT | N CoEvoRDT | CoEvoRDT + N FPRDT | N CoEvoRDT + N FPRDT | N FPRDT | N CoEvoRDT | CoEvoRDT + N FPRDT | N CoEvoRDT + N FPRDT | N FPRDT | N CoEvoRDT | CoEvoRDT + N FPRDT | N CoEvoRDT + N FPRDT |
| 1 | .124 | .113 | .112 | .112 | .742 | .745 | .754 | .754 | 21 | 68 | 93 | 93 |
| 2 | .123 | .112 | .112 | .111 | .746 | .751 | .757 | .760 | 45 | 144 | 120 | 177 |
| 3 | .122 | .112 | .111 | .111 | .749 | .752 | .768 | .772 | 62 | 211 | 191 | 280 |
| 4 | .122 | .111 | .111 | .111 | .761 | .769 | .770 | .778 | 81 | 286 | 236 | 356 |
| 5 | .121 | .111 | .110 | .110 | .764 | .780 | .780 | .784 | 112 | 345 | 310 | 421 |
| 10 | .118 | .109 | .108 | .108 | .782 | .789 | .792 | .798 | 213 | 745 | 394 | 878 |
| 20 | .112 | .106 | .105 | .104 | .792 | .798 | .806 | .809 | 410 | 1252 | 784 | 1790 |
| 50 | .112 | .105 | .104 | .104 | .800 | .802 | .807 | .805 | 973 | 3138 | 1341 | 4361 |
| 100 | .111 | .105 | .105 | .104 | .795 | .807 | .806 | .810 | 2210 | 6755 | 3241 | 8561 |

Table 15: Best results of repeated $N$ algorithms' runs for **mnist** dataset. In CoEvoRDT + N FPRDT resulting DTs from N independent runs of FPRDT were incorporated into CoEvoRDT initial population.

|   | minimax regret | | | | adversarial accuracy | | | | computation time [s] | | | |
|---|---|---|---|---|---|---|---|---|---|---|---|---|
| N | N FPRDT | N CoEvoRDT | CoEvoRDT + N FPRDT | N CoEvoRDT + N FPRDT | N FPRDT | N CoEvoRDT | CoEvoRDT + N FPRDT | N CoEvoRDT + N FPRDT | N FPRDT | N CoEvoRDT | CoEvoRDT + N FPRDT | N CoEvoRDT + N FPRDT |
| 1 | .238 | .196 | .196 | .196 | .978 | .982 | .982 | .982 | 8 | 23 | 34 | 34 |
| 2 | .236 | .195 | .195 | .196 | .979 | .983 | .984 | .984 | 17 | 51 | 40 | 61 |
| 3 | .234 | .195 | .195 | .195 | .982 | .984 | .984 | .984 | 26 | 66 | 71 | 84 |
| 4 | .233 | .194 | .193 | .192 | .982 | .984 | .984 | .985 | 32 | 91 | 80 | 114 |
| 5 | .232 | .192 | .193 | .192 | .982 | .984 | .985 | .985 | 37 | 116 | 98 | 153 |
| 10 | .227 | .189 | .189 | .189 | .983 | .985 | .985 | .987 | 87 | 220 | 145 | 282 |
| 20 | .215 | .183 | .183 | .183 | .983 | .985 | .987 | .987 | 176 | 479 | 258 | 716 |
| 50 | .215 | .183 | .182 | .182 | .984 | .987 | .987 | .987 | 431 | 1037 | 475 | 1527 |
| 100 | .214 | .181 | .183 | .183 | .985 | .987 | .987 | .988 | 843 | 2496 | 1091 | 3381 |

Table 16: Best results of repeated $N$ algorithms' runs for **F-mnist2v5** dataset. In CoEvoRDT + N FPRDT resulting DTs from N independent runs of FPRDT were incorporated into CoEvoRDT initial population.

|   | minimax regret | | | | adversarial accuracy | | | | computation time [s] | | | |
|---|---|---|---|---|---|---|---|---|---|---|---|---|
| N | N FPRDT | N CoEvoRDT | CoEvoRDT + N FPRDT | N CoEvoRDT + N FPRDT | N FPRDT | N CoEvoRDT | CoEvoRDT + N FPRDT | N CoEvoRDT + N FPRDT | N FPRDT | N CoEvoRDT | CoEvoRDT + N FPRDT | N CoEvoRDT + N FPRDT |
| 1 | .232 | .202 | .199 | .199 | .865 | .869 | .870 | .870 | 9 | 25 | 36 | 36 |
| 2 | .230 | .201 | .197 | .198 | .874 | .873 | .875 | .882 | 16 | 50 | 45 | 61 |
| 3 | .229 | .201 | .197 | .197 | .877 | .885 | .880 | .884 | 25 | 68 | 65 | 90 |
| 4 | .227 | .199 | .197 | .196 | .887 | .892 | .893 | .890 | 39 | 97 | 82 | 135 |
| 5 | .226 | .199 | .196 | .196 | .891 | .906 | .900 | .904 | 43 | 131 | 117 | 158 |
| 10 | .222 | .196 | .192 | .191 | .915 | .920 | .913 | .922 | 81 | 261 | 157 | 368 |
| 20 | .210 | .189 | .186 | .185 | .927 | .930 | .929 | .930 | 163 | 452 | 298 | 631 |
| 50 | .208 | .187 | .186 | .185 | .929 | .932 | .928 | .933 | 474 | 1235 | 498 | 1739 |
| 100 | .209 | .187 | .186 | .185 | .933 | .939 | .927 | .937 | 870 | 2748 | 1302 | 3919 |

Table 17: Best results of repeated $N$ algorithms' runs for **F-mnist3v4** dataset. In CoEvoRDT + N FPRDT resulting DTs from N independent runs of FPRDT were incorporated into CoEvoRDT initial population.

|   | minimax regret | | | | adversarial accuracy | | | | computation time [s] | | | |
|---|---|---|---|---|---|---|---|---|---|---|---|---|
| N | N FPRDT | N CoEvoRDT | CoEvoRDT + N FPRDT | N CoEvoRDT + N FPRDT | N FPRDT | N CoEvoRDT | CoEvoRDT + N FPRDT | N CoEvoRDT + N FPRDT | N FPRDT | N CoEvoRDT | CoEvoRDT + N FPRDT | N CoEvoRDT + N FPRDT |
| 1 | .240 | .208 | .207 | .207 | .876 | .868 | .880 | .880 | 9 | 26 | 35 | 35 |
| 2 | .238 | .206 | .205 | .206 | .885 | .878 | .890 | .885 | 19 | 51 | 47 | 71 |
| 3 | .235 | .205 | .204 | .204 | .886 | .884 | .888 | .897 | 28 | 76 | 74 | 100 |
| 4 | .235 | .206 | .204 | .204 | .892 | .896 | .903 | .905 | 37 | 95 | 95 | 126 |
| 5 | .234 | .205 | .203 | .204 | .905 | .911 | .909 | .918 | 42 | 129 | 118 | 161 |
| 10 | .227 | .200 | .201 | .200 | .920 | .916 | .924 | .929 | 83 | 245 | 138 | 355 |
| 20 | .217 | .195 | .193 | .192 | .933 | .933 | .937 | .937 | 166 | 501 | 286 | 658 |
| 50 | .216 | .192 | .194 | .193 | .938 | .934 | .942 | .945 | 430 | 1312 | 535 | 1657 |
| 100 | .216 | .192 | .193 | .192 | .942 | .942 | .944 | .939 | 878 | 2825 | 1270 | 3528 |

Table 18: Best results of repeated $N$ algorithms' runs for **F-mnist7v9** dataset. In CoEvoRDT + N FPRDT resulting DTs from N independent runs of FPRDT were incorporated into CoEvoRDT initial population.

|   | minimax regret | | | | adversarial accuracy | | | | computation time [s] | | | |
|---|---|---|---|---|---|---|---|---|---|---|---|---|
| N | N FPRDT | N CoEvoRDT | CoEvoRDT + N FPRDT | N CoEvoRDT + N FPRDT | N FPRDT | N CoEvoRDT | CoEvoRDT + N FPRDT | N CoEvoRDT + N FPRDT | N FPRDT | N CoEvoRDT | CoEvoRDT + N FPRDT | N CoEvoRDT + N FPRDT |
| 1 | .314 | .241 | .238 | .238 | .678 | .685 | .693 | .692 | 40 | 146 | 191 | 178 |
| 2 | .310 | .239 | .236 | .235 | .683 | .695 | .702 | .704 | 82 | 282 | 230 | 338 |
| 3 | .308 | .237 | .234 | .234 | .688 | .699 | .711 | .719 | 120 | 429 | 271 | 534 |
| 4 | .306 | .236 | .233 | .233 | .699 | .714 | .718 | .723 | 158 | 599 | 297 | 747 |
| 5 | .304 | .235 | .232 | .232 | .714 | .732 | .734 | .735 | 199 | 730 | 347 | 865 |
| 10 | .294 | .229 | .226 | .225 | .745 | .746 | .753 | .759 | 406 | 1478 | 560 | 1729 |
| 20 | .276 | .218 | .215 | .214 | .764 | .765 | .772 | .774 | 808 | 2926 | 956 | 3513 |
| 50 | .272 | .216 | .214 | .214 | .766 | .768 | .772 | .775 | 2095 | 7662 | 2234 | 9278 |
| 100 | .272 | .216 | .214 | .214 | .766 | - | .774 | - | 4008 | >10000 | 4179 | >10000 |

Table 19: Best results of repeated $N$ algorithms' runs for **cifar10:0v5** dataset. In CoEvoRDT + N FPRDT resulting DTs from N independent runs of FPRDT were incorporated into CoEvoRDT initial population.

|   | minimax regret | | | | adversarial accuracy | | | | computation time [s] | | | |
|---|---|---|---|---|---|---|---|---|---|---|---|---|
| N | N FPRDT | N CoEvoRDT | CoEvoRDT + N FPRDT | N CoEvoRDT + N FPRDT | N FPRDT | N CoEvoRDT | CoEvoRDT + N FPRDT | N CoEvoRDT + N FPRDT | N FPRDT | N CoEvoRDT | CoEvoRDT + N FPRDT | N CoEvoRDT + N FPRDT |
| 1 | .341 | .289 | .289 | .289 | .688 | .692 | .697 | .697 | 42 | 126 | 174 | 174 |
| 2 | .338 | .286 | .287 | .287 | .692 | .694 | .705 | .704 | 77 | 255 | 200 | 336 |
| 3 | .336 | .286 | .285 | .286 | .698 | .700 | .710 | .709 | 130 | 364 | 315 | 497 |
| 4 | .336 | .284 | .286 | .285 | .702 | .712 | .716 | .714 | 163 | 455 | 450 | 577 |
| 5 | .334 | .285 | .284 | .283 | .709 | .725 | .720 | .722 | 218 | 629 | 552 | 804 |
| 10 | .324 | .279 | .278 | .280 | .727 | .731 | .730 | .738 | 453 | 1366 | 671 | 1865 |
| 20 | .309 | .270 | .271 | .271 | .737 | .744 | .743 | .742 | 910 | 2639 | 1503 | 3563 |
| 50 | .307 | .268 | .269 | .270 | .738 | .744 | .745 | .744 | 2003 | 6210 | 2498 | 7412 |
| 100 | .306 | .267 | .270 | .267 | .742 | - | .749 | - | 3897 | >10000 | 6155 | >10000 |

Table 20: Best results of repeated $N$ algorithms' runs for **cifar10:0v6** dataset. In CoEvoRDT + N FPRDT resulting DTs from N independent runs of FPRDT were incorporated into CoEvoRDT initial population.

|   | minimax regret | | | | adversarial accuracy | | | | computation time [s] | | | |
|---|---|---|---|---|---|---|---|---|---|---|---|---|
| N | N FPRDT | N CoEvoRDT | CoEvoRDT + N FPRDT | N CoEvoRDT + N FPRDT | N FPRDT | N CoEvoRDT | CoEvoRDT + N FPRDT | N CoEvoRDT + N FPRDT | N FPRDT | N CoEvoRDT | CoEvoRDT + N FPRDT | N CoEvoRDT + N FPRDT |
| 1 | .331 | .283 | .281 | .281 | .661 | .663 | .664 | .664 | 41 | 111 | 163 | 163 |
| 2 | .327 | .280 | .279 | .280 | .665 | .667 | .669 | .668 | 74 | 209 | 202 | 289 |
| 3 | .327 | .281 | .279 | .279 | .669 | .672 | .675 | .675 | 125 | 306 | 291 | 397 |
| 4 | .324 | .279 | .278 | .278 | .678 | .683 | .679 | .682 | 180 | 403 | 356 | 604 |
| 5 | .324 | .279 | .277 | .277 | .684 | .691 | .690 | .692 | 218 | 523 | 446 | 799 |
| 10 | .315 | .273 | .273 | .271 | .697 | .699 | .699 | .705 | 436 | 1158 | 625 | 1699 |
| 20 | .301 | .263 | .264 | .263 | .709 | .713 | .708 | .713 | 888 | 2399 | 1131 | 3237 |
| 50 | .296 | .263 | .261 | .260 | .710 | .713 | .708 | .708 | 1888 | 5628 | 2314 | 8033 |
| 100 | .298 | .261 | .263 | .261 | .714 | - | .711 | - | 4039 | >10000 | 6193 | >10000 |

Table 21: Best results of repeated $N$ algorithms' runs for **cifar10:4v8** dataset. In CoEvoRDT + N FPRDT resulting DTs from N independent runs of FPRDT were incorporated into CoEvoRDT initial population.

|   | minimax regret | | | | | adversarial accuracy | | | | | computation time [s] | | | | |
|---|---|---|---|---|---|---|---|---|---|---|---|---|---|---|---|
| HoF size | Nash mixed tree | Top K as mixed tree | Nash single trees | Top K | Best | Nash mixed tree | Top K as mixed tree | Nash single trees | Top K | Best | Nash mixed tree | Top K as mixed tree | Nash single trees | Top K | Best |
| 0 | .058 | .058 | .058 | .058 | .059 | .778 | .778 | .778 | .778 | .778 | 0.6 | 0.6 | 0.6 | 0.6 | 0.6 |
| 10 | .053 | .055 | .056 | .056 | .057 | .786 | .784 | .783 | .781 | .780 | 0.7 | 0.6 | 0.6 | 0.6 | 0.6 |
| 20 | .053 | .055 | .055 | .055 | .057 | .786 | .785 | .784 | .782 | .781 | 0.7 | 0.7 | 0.7 | 0.7 | 0.7 |
| 50 | .053 | .054 | .055 | .055 | .056 | .787 | .785 | .784 | .783 | .782 | 0.8 | 0.7 | 0.7 | 0.8 | 0.7 |
| 100 | .052 | .053 | .054 | .054 | .056 | .789 | .787 | .786 | .784 | .782 | 0.9 | 0.8 | 0.8 | 0.9 | 0.7 |
| 200 | .052 | .053 | .054 | .054 | .055 | .791 | .788 | .787 | .784 | .782 | 1.0 | 0.9 | 0.9 | 1.0 | 0.9 |
| 500 | .052 | .053 | .054 | .054 | .055 | .791 | .789 | .788 | .785 | .782 | 1.2 | 1.0 | 1.0 | 1.2 | 0.9 |
| ∞ | .052 | .052 | .054 | .053 | .054 | .792 | .790 | .789 | .787 | .783 | 1.2 | 1.0 | 1.0 | 1.2 | 0.9 |

Table 22: Results with respect of HoF size for **ionos** dataset. ∞ means that there was no limit no limit on HoF size.

|          | minimax regret ||||| adversarial accuracy ||||| computation time [s] |||||
|----------|-------------|-------------|-------------|-------|------|-------------|-------------|-------------|-------|------|-------------|-------------|-------------|-------|------|
| HoF size | Nash mixed tree | Top K as mixed tree | Nash single trees | Top K | Best | Nash mixed tree | Top K as mixed tree | Nash single trees | Top K | Best | Nash mixed tree | Top K as mixed tree | Nash single trees | Top K | Best |
| 0   | .054 | .055 | .055 | .055 | .055 | .870 | .870 | .870 | .870 | .870 | 1.2 | 1.2 | 1.2 | 1.2 | 1.2 |
| 10  | .050 | .051 | .052 | .052 | .054 | .879 | .878 | .876 | .874 | .873 | 1.4 | 1.3 | 1.3 | 1.2 | 1.2 |
| 20  | .050 | .051 | .052 | .052 | .053 | .880 | .878 | .877 | .875 | .874 | 1.6 | 1.3 | 1.3 | 1.5 | 1.4 |
| 50  | .050 | .051 | .052 | .052 | .053 | .880 | .879 | .877 | .876 | .875 | 1.7 | 1.5 | 1.3 | 1.6 | 1.4 |
| 100 | .049 | .050 | .051 | .051 | .052 | .883 | .881 | .879 | .877 | .875 | 1.6 | 1.5 | 1.5 | 1.7 | 1.4 |
| 200 | .049 | .050 | .051 | .051 | .052 | .885 | .882 | .881 | .877 | .875 | 2.0 | 1.8 | 1.8 | 2.1 | 1.8 |
| 500 | .049 | .050 | .051 | .051 | .052 | .885 | .883 | .881 | .878 | .875 | 2.3 | 2.1 | 2.2 | 2.5 | 2.0 |
| Inf | .049 | .049 | .051 | .050 | .051 | .886 | .883 | .883 | .880 | .876 | 2.3 | 2.2 | 2.2 | 2.6 | 2.1 |

Table 23: Results with respect of HoF size for **breast** dataset. ∞ means that there was no limit no limit on HoF size.

|          | minimax regret ||||| adversarial accuracy ||||| computation time [s] |||||
|----------|-------------|-------------|-------------|-------|------|-------------|-------------|-------------|-------|------|-------------|-------------|-------------|-------|------|
| HoF size | Nash mixed tree | Top K as mixed tree | Nash single trees | Top K | Best | Nash mixed tree | Top K as mixed tree | Nash single trees | Top K | Best | Nash mixed tree | Top K as mixed tree | Nash single trees | Top K | Best |
| 0   | .103 | .108 | .108 | .108 | .108 | .607 | .607 | .607 | .607 | .607 | 1.6 | 1.6 | 1.6 | 1.6 | 1.6 |
| 10  | .098 | .101 | .103 | .102 | .106 | .613 | .612 | .611 | .609 | .608 | 2.0 | 1.9 | 1.8 | 1.9 | 1.8 |
| 20  | .098 | .101 | .102 | .102 | .104 | .613 | .612 | .611 | .610 | .609 | 2.2 | 2.1 | 1.8 | 2.0 | 1.9 |
| 50  | .097 | .099 | .101 | .101 | .104 | .614 | .613 | .611 | .611 | .610 | 2.6 | 2.1 | 2.0 | 2.5 | 2.0 |
| 100 | .096 | .098 | .100 | .100 | .103 | .615 | .614 | .613 | .611 | .610 | 2.5 | 2.6 | 2.2 | 2.6 | 2.4 |
| 200 | .096 | .098 | .099 | .099 | .101 | .617 | .615 | .614 | .612 | .610 | 3.0 | 2.8 | 2.5 | 2.9 | 2.7 |
| 500 | .096 | .098 | .099 | .099 | .101 | .617 | .615 | .615 | .612 | .610 | 3.5 | 3.3 | 3.2 | 3.3 | 2.6 |
| ∞   | .095 | .096 | .099 | .097 | .100 | .618 | .616 | .615 | .614 | .611 | 3.3 | 3.4 | 3.2 | 3.6 | 3.0 |

Table 24: Results with respect of HoF size for **diabetes** dataset. ∞ means that there was no limit no limit on HoF size.

|          | minimax regret ||||| adversarial accuracy ||||| computation time [s] |||||
|----------|-------------|-------------|-------------|-------|------|-------------|-------------|-------------|-------|------|-------------|-------------|-------------|-------|------|
| HoF size | Nash mixed tree | Top K as mixed tree | Nash single trees | Top K | Best | Nash mixed tree | Top K as mixed tree | Nash single trees | Top K | Best | Nash mixed tree | Top K as mixed tree | Nash single trees | Top K | Best |
| 0   | .083 | .085 | .085 | .085 | .086 | .646 | .646 | .646 | .646 | .646 | 3.3 | 3.3 | 3.3 | 3.3 | 3.3 |
| 10  | .078 | .080 | .081 | .081 | .084 | .653 | .651 | .650 | .649 | .648 | 4.1 | 3.7 | 3.8 | 3.6 | 3.9 |
| 20  | .078 | .080 | .080 | .081 | .083 | .653 | .652 | .651 | .649 | .649 | 4.2 | 3.7 | 4.0 | 4.7 | 4.0 |
| 50  | .077 | .079 | .080 | .080 | .082 | .654 | .652 | .651 | .650 | .649 | 5.4 | 4.8 | 4.1 | 5.0 | 4.0 |
| 100 | .076 | .078 | .079 | .079 | .081 | .655 | .654 | .653 | .651 | .649 | 5.3 | 4.5 | 4.2 | 5.0 | 4.6 |
| 200 | .076 | .077 | .079 | .079 | .080 | .657 | .655 | .654 | .651 | .650 | 6.0 | 5.0 | 5.4 | 6.7 | 5.0 |
| 500 | .076 | .077 | .079 | .078 | .080 | .657 | .655 | .654 | .652 | .649 | 6.5 | 6.5 | 6.8 | 7.1 | 5.1 |
| ∞   | .075 | .076 | .078 | .077 | .079 | .658 | .656 | .655 | .654 | .650 | 7.5 | 6.8 | 6.8 | 8.0 | 5.6 |

Table 25: Results with respect of HoF size for **bank** dataset. ∞ means that there was no limit no limit on HoF size.

|          | minimax regret ||||| adversarial accuracy ||||| computation time [s] |||||
|----------|-------------|-------------|-------------|-------|------|-------------|-------------|-------------|-------|------|-------------|-------------|-------------|-------|------|
| HoF size | Nash mixed tree | Top K as mixed tree | Nash single trees | Top K | Best | Nash mixed tree | Top K as mixed tree | Nash single trees | Top K | Best | Nash mixed tree | Top K as mixed tree | Nash single trees | Top K | Best |
| 0   | .066 | .070 | .070 | .070 | .070 | .654 | .654 | .654 | .654 | .654 | 5.2 | 5.2 | 5.2 | 5.2 | 5.2 |
| 10  | .063 | .065 | .066 | .066 | .068 | .661 | .659 | .658 | .657 | .656 | 5.9 | 5.9 | 5.1 | 6.0 | 5.5 |
| 20  | .063 | .065 | .066 | .066 | .067 | .661 | .660 | .659 | .657 | .657 | 6.0 | 6.6 | 6.3 | 6.3 | 5.9 |
| 50  | .063 | .064 | .066 | .065 | .067 | .662 | .660 | .659 | .658 | .657 | 7.8 | 6.9 | 6.3 | 7.5 | 5.9 |
| 100 | .062 | .064 | .065 | .065 | .066 | .663 | .662 | .661 | .659 | .657 | 7.4 | 7.5 | 6.6 | 7.4 | 7.2 |
| 200 | .062 | .063 | .064 | .064 | .066 | .665 | .663 | .662 | .659 | .658 | 9.1 | 8.9 | 7.6 | 10.0 | 7.9 |
| 500 | .062 | .063 | .064 | .064 | .065 | .665 | .663 | .662 | .660 | .657 | 10.5 | 9.6 | 10.0 | 12.0 | 8.9 |
| ∞   | .062 | .062 | .064 | .063 | .065 | .666 | .664 | .663 | .662 | .658 | 10.0 | 9.6 | 10.2 | 12.1 | 9.0 |

Table 26: Results with respect of HoF size for **Japan3v4** dataset. ∞ means that there was no limit no limit on HoF size.

|          | minimax regret ||||| adversarial accuracy ||||| computation time [s] |||||
|----------|-------------|-------------|-------------|-------|------|-------------|-------------|-------------|-------|------|-------------|-------------|-------------|-------|------|
| HoF size | Nash mixed tree | Top K as mixed tree | Nash single trees | Top K | Best | Nash mixed tree | Top K as mixed tree | Nash single trees | Top K | Best | Nash mixed tree | Top K as mixed tree | Nash single trees | Top K | Best |
| 0   | .081 | .079 | .079 | .078 | .079 | .739 | .739 | .739 | .739 | .739 | 7.1 | 7.1 | 7.1 | 7.1 | 7.1 |
| 10  | .072 | .073 | .075 | .075 | .077 | .746 | .745 | .743 | .742 | .741 | 8.3 | 8.4 | 8.1 | 9.1 | 7.7 |
| 20  | .072 | .073 | .074 | .074 | .076 | .747 | .745 | .744 | .742 | .741 | 9.5 | 9.7 | 9.0 | 10.2 | 9.7 |
| 50  | .071 | .073 | .074 | .074 | .076 | .747 | .746 | .744 | .743 | .742 | 11.0 | 10.4 | 8.9 | 9.7 | 10.1 |
| 100 | .070 | .072 | .073 | .073 | .075 | .749 | .747 | .746 | .744 | .742 | 11.5 | 10.1 | 10.1 | 11.7 | 10.4 |
| 200 | .070 | .071 | .073 | .072 | .074 | .751 | .749 | .747 | .745 | .743 | 13.2 | 12.7 | 10.8 | 14.3 | 11.1 |
| 500 | .070 | .071 | .072 | .072 | .074 | .751 | .749 | .748 | .745 | .742 | 15.7 | 13.5 | 13.7 | 15.8 | 12.5 |
| ∞   | .070 | .070 | .072 | .071 | .073 | .752 | .750 | .749 | .747 | .743 | 17.1 | 13.5 | 13.9 | 15.9 | 12.8 |

Table 27: Results with respect of HoF size for **spam** dataset. ∞ means that there was no limit no limit on HoF size.

| HoF size | minimax regret | | | | | adversarial accuracy | | | | | computation time [s] | | | | |
|---|---|---|---|---|---|---|---|---|---|---|---|---|---|---|---|
| | Nash mixed tree | Top K as mixed tree | Nash single trees | Top K | Best | Nash mixed tree | Top K as mixed tree | Nash single trees | Top K | Best | Nash mixed tree | Top K as mixed tree | Nash single trees | Top K | Best |
| 0 | .122 | .128 | .128 | .128 | .128 | .728 | .728 | .728 | .728 | .728 | 6.1 | 6.1 | 6.1 | 6.1 | 6.1 |
| 10 | .117 | .120 | .122 | .122 | .125 | .735 | .734 | .732 | .731 | .730 | 6.9 | 7.8 | 7.3 | 8.0 | 6.3 |
| 20 | .117 | .120 | .121 | .121 | .124 | .736 | .734 | .733 | .731 | .731 | 8.5 | 7.9 | 7.2 | 9.2 | 8.3 |
| 50 | .115 | .118 | .120 | .120 | .123 | .736 | .735 | .733 | .732 | .731 | 8.8 | 8.2 | 8.2 | 9.8 | 7.4 |
| 100 | .114 | .117 | .119 | .119 | .122 | .738 | .737 | .735 | .733 | .731 | 10.6 | 9.6 | 8.2 | 10.2 | 8.7 |
| 200 | .114 | .116 | .118 | .118 | .120 | .740 | .738 | .736 | .734 | .732 | 11.5 | 10.1 | 11.0 | 12.0 | 10.5 |
| 500 | .114 | .116 | .118 | .118 | .120 | .740 | .738 | .737 | .734 | .731 | 12.7 | 12.3 | 12.4 | 14.2 | 11.8 |
| ∞ | .113 | .114 | .118 | .116 | .119 | .741 | .739 | .738 | .736 | .732 | 13.9 | 12.6 | 12.5 | 14.3 | 12.2 |

Table 28: Results with respect of HoF size for **GesDvP** dataset. ∞ means that there was no limit no limit on HoF size.

| HoF size | minimax regret | | | | | adversarial accuracy | | | | | computation time [s] | | | | |
|---|---|---|---|---|---|---|---|---|---|---|---|---|---|---|---|
| | Nash mixed tree | Top K as mixed tree | Nash single trees | Top K | Best | Nash mixed tree | Top K as mixed tree | Nash single trees | Top K | Best | Nash mixed tree | Top K as mixed tree | Nash single trees | Top K | Best |
| 0 | .075 | .072 | .072 | .072 | .072 | .804 | .804 | .804 | .804 | .804 | 7.0 | 7.0 | 7.0 | 7.0 | 7.0 |
| 10 | .065 | .067 | .068 | .068 | .070 | .812 | .811 | .810 | .808 | .807 | 7.8 | 8.0 | 7.7 | 7.4 | 8.0 |
| 20 | .065 | .067 | .068 | .068 | .070 | .813 | .812 | .811 | .809 | .808 | 8.6 | 7.7 | 7.8 | 9.1 | 7.6 |
| 50 | .065 | .066 | .068 | .068 | .069 | .814 | .812 | .811 | .810 | .808 | 9.4 | 8.8 | 9.3 | 8.7 | 8.4 |
| 100 | .064 | .066 | .067 | .067 | .069 | .816 | .814 | .813 | .811 | .808 | 11.6 | 9.7 | 8.9 | 10.2 | 9.7 |
| 200 | .064 | .065 | .066 | .066 | .068 | .818 | .815 | .814 | .811 | .809 | 12.1 | 10.2 | 9.9 | 11.3 | 11.1 |
| 500 | .064 | .065 | .066 | .066 | .068 | .818 | .816 | .815 | .812 | .809 | 15.4 | 12.3 | 13.1 | 15.4 | 10.5 |
| ∞ | .064 | .064 | .066 | .065 | .067 | .819 | .817 | .816 | .814 | .809 | 14.1 | 12.4 | 13.2 | 15.7 | 12.5 |

Table 29: Results with respect of HoF size for **har1v2** dataset. ∞ means that there was no limit no limit on HoF size.

| HoF size | minimax regret | | | | | adversarial accuracy | | | | | computation time [s] | | | | |
|---|---|---|---|---|---|---|---|---|---|---|---|---|---|---|---|
| | Nash mixed tree | Top K as mixed tree | Nash single trees | Top K | Best | Nash mixed tree | Top K as mixed tree | Nash single trees | Top K | Best | Nash mixed tree | Top K as mixed tree | Nash single trees | Top K | Best |
| 0 | .105 | .101 | .101 | .101 | .101 | .677 | .677 | .677 | .677 | .677 | 1.3 | 1.3 | 1.3 | 1.3 | 1.3 |
| 10 | .092 | .094 | .096 | .096 | .099 | .683 | .682 | .681 | .679 | .678 | 1.4 | 1.5 | 1.3 | 1.4 | 1.5 |
| 20 | .092 | .094 | .095 | .095 | .098 | .684 | .683 | .682 | .680 | .679 | 1.6 | 1.4 | 1.5 | 1.7 | 1.6 |
| 50 | .091 | .093 | .095 | .095 | .097 | .684 | .683 | .682 | .681 | .680 | 1.8 | 1.6 | 1.6 | 1.8 | 1.6 |
| 100 | .090 | .092 | .094 | .094 | .096 | .686 | .685 | .683 | .682 | .680 | 2.0 | 1.7 | 1.9 | 1.9 | 1.5 |
| 200 | .090 | .092 | .093 | .093 | .095 | .688 | .686 | .685 | .682 | .681 | 2.3 | 1.9 | 2.0 | 2.3 | 1.9 |
| 500 | .090 | .092 | .093 | .093 | .095 | .688 | .686 | .685 | .683 | .680 | 2.5 | 2.3 | 2.2 | 2.9 | 2.1 |
| ∞ | .089 | .090 | .093 | .091 | .094 | .689 | .687 | .686 | .684 | .681 | 2.9 | 2.5 | 2.4 | 3.0 | 2.2 |

Table 30: Results with respect of HoF size for **wine** dataset. ∞ means that there was no limit no limit on HoF size.

| HoF size | minimax regret | | | | | adversarial accuracy | | | | | computation time [s] | | | | |
|---|---|---|---|---|---|---|---|---|---|---|---|---|---|---|---|
| | Nash mixed tree | Top K as mixed tree | Nash single trees | Top K | Best | Nash mixed tree | Top K as mixed tree | Nash single trees | Top K | Best | Nash mixed tree | Top K as mixed tree | Nash single trees | Top K | Best |
| 0 | .066 | .068 | .068 | .068 | .069 | .785 | .785 | .785 | .785 | .785 | 9.8 | 9.8 | 9.8 | 9.8 | 9.8 |
| 10 | .062 | .064 | .065 | .065 | .067 | .793 | .791 | .790 | .788 | .787 | 10.1 | 9.9 | 10.7 | 11.5 | 9.4 |
| 20 | .062 | .064 | .065 | .065 | .066 | .793 | .792 | .791 | .789 | .788 | 12.5 | 12.2 | 11.8 | 13.0 | 11.5 |
| 50 | .062 | .063 | .064 | .064 | .066 | .794 | .792 | .791 | .790 | .789 | 12.8 | 11.9 | 11.2 | 11.9 | 11.7 |
| 100 | .061 | .063 | .064 | .064 | .065 | .796 | .794 | .793 | .791 | .789 | 14.4 | 14.5 | 12.3 | 13.7 | 12.6 |
| 200 | .061 | .062 | .063 | .063 | .064 | .798 | .795 | .794 | .791 | .789 | 16.8 | 16.6 | 14.4 | 16.3 | 15.2 |
| 500 | .061 | .062 | .063 | .063 | .064 | .798 | .796 | .795 | .792 | .789 | 17.9 | 17.0 | 15.7 | 18.7 | 14.7 |
| ∞ | .061 | .061 | .063 | .062 | .064 | .799 | .797 | .796 | .794 | .790 | 18.6 | 17.3 | 18.2 | 18.9 | 16.6 |

Table 31: Results with respect of HoF size for **collision-det** dataset. ∞ means that there was no limit no limit on HoF size.

| HoF size | minimax regret | | | | | adversarial accuracy | | | | | computation time [s] | | | | |
|---|---|---|---|---|---|---|---|---|---|---|---|---|---|---|---|
| | Nash mixed tree | Top K as mixed tree | Nash single trees | Top K | Best | Nash mixed tree | Top K as mixed tree | Nash single trees | Top K | Best | Nash mixed tree | Top K as mixed tree | Nash single trees | Top K | Best |
| 0 | .062 | .062 | .062 | .062 | .062 | .948 | .948 | .948 | .948 | .948 | 9.6 | 9.6 | 9.6 | 9.6 | 9.6 |
| 10 | .056 | .058 | .059 | .059 | .060 | .958 | .956 | .954 | .952 | .951 | 12.3 | 10.3 | 11.4 | 10.8 | 10.5 |
| 20 | .056 | .058 | .058 | .058 | .060 | .958 | .957 | .955 | .953 | .952 | 13.5 | 11.5 | 10.7 | 12.4 | 12.0 |
| 50 | .056 | .057 | .058 | .058 | .060 | .959 | .957 | .955 | .954 | .953 | 15.6 | 12.6 | 12.6 | 13.1 | 11.2 |
| 100 | .055 | .056 | .057 | .057 | .059 | .961 | .959 | .958 | .955 | .953 | 17.3 | 15.3 | 14.6 | 14.0 | 13.0 |
| 200 | .055 | .056 | .057 | .057 | .058 | .964 | .961 | .959 | .956 | .954 | 17.5 | 16.0 | 14.4 | 17.8 | 16.7 |
| 500 | .055 | .056 | .057 | .057 | .058 | .964 | .962 | .960 | .956 | .953 | 21.5 | 17.0 | 19.2 | 21.0 | 14.9 |
| ∞ | .055 | .055 | .057 | .056 | .057 | .965 | .962 | .961 | .959 | .954 | 19.6 | 18.6 | 19.8 | 21.2 | 17.0 |

Table 32: Results with respect of HoF size for **mnist:v1-5** dataset. ∞ means that there was no limit no limit on HoF size.

| HoF size | minimax regret ||||| adversarial accuracy ||||| computation time [s] |||||
|---|---|---|---|---|---|---|---|---|---|---|---|---|---|---|---|
| | Nash mixed tree | Top K as mixed tree | Nash single trees | Top K | Best | Nash mixed tree | Top K as mixed tree | Nash single trees | Top K | Best | Nash mixed tree | Top K as mixed tree | Nash single trees | Top K | Best |
| 0 | .062 | .062 | .062 | .062 | .062 | .912 | .912 | .912 | .912 | .912 | 81.2 | 81.2 | 81.2 | 81.2 | 81.2 |
| 10 | .056 | .057 | .057 | .058 | .061 | .915 | .915 | .914 | .914 | .912 | 91.1 | 84.7 | 87.8 | 86.0 | 84.0 |
| 20 | .057 | .058 | .057 | .058 | .060 | .915 | .915 | .915 | .915 | .913 | 95.0 | 88.1 | 97.5 | 97.8 | 91.2 |
| 50 | .055 | .057 | .057 | .058 | .060 | .916 | .916 | .915 | .915 | .912 | 116.2 | 91.5 | 104.4 | 100.5 | 102.2 |
| 100 | .055 | .056 | .056 | .057 | .059 | .917 | .918 | .917 | .917 | .914 | 125.1 | 94.7 | 123.3 | 114.7 | 110.6 |
| 200 | .055 | .055 | .055 | .056 | .058 | .917 | .917 | .916 | .916 | .914 | 132.3 | 102.1 | 132.6 | 139.0 | 124.1 |
| 500 | .054 | .055 | .055 | .056 | .058 | .917 | .917 | .917 | .916 | .915 | 154.2 | 114.5 | 163.2 | 173.8 | 144.9 |
| ∞ | .054 | .055 | .055 | .055 | .058 | .917 | .917 | .917 | .917 | .915 | 159.4 | 134.0 | 179.5 | 177.2 | 153.0 |

Table 33: Results with respect of HoF size for **mnist:2v6** dataset. ∞ means that there was no limit on HoF size.

| HoF size | minimax regret ||||| adversarial accuracy ||||| computation time [s] |||||
|---|---|---|---|---|---|---|---|---|---|---|---|---|---|---|---|
| | Nash mixed tree | Top K as mixed tree | Nash single trees | Top K | Best | Nash mixed tree | Top K as mixed tree | Nash single trees | Top K | Best | Nash mixed tree | Top K as mixed tree | Nash single trees | Top K | Best |
| 0 | .121 | .127 | .127 | .127 | .127 | .733 | .733 | .733 | .733 | .733 | 38.7 | 38.7 | 38.7 | 38.7 | 38.7 |
| 10 | .116 | .119 | .121 | .121 | .124 | .740 | .739 | .737 | .736 | .735 | 44.0 | 41.7 | 39.0 | 43.1 | 41.5 |
| 20 | .116 | .118 | .120 | .120 | .123 | .741 | .739 | .738 | .736 | .735 | 49.6 | 45.3 | 47.9 | 48.8 | 46.5 |
| 50 | .114 | .117 | .119 | .119 | .122 | .741 | .740 | .738 | .737 | .736 | 51.6 | 52.6 | 51.3 | 49.6 | 49.4 |
| 100 | .113 | .116 | .118 | .118 | .121 | .743 | .742 | .740 | .738 | .736 | 67.1 | 53.7 | 50.1 | 64.7 | 55.2 |
| 200 | .113 | .115 | .117 | .117 | .119 | .745 | .743 | .741 | .739 | .737 | 68.1 | 66.9 | 54.3 | 73.5 | 60.6 |
| 500 | .113 | .115 | .117 | .117 | .119 | .745 | .743 | .742 | .739 | .736 | 74.3 | 68.5 | 70.7 | 79.8 | 64.5 |
| ∞ | .112 | .113 | .117 | .115 | .118 | .746 | .744 | .743 | .741 | .737 | 82.1 | 69.2 | 71.1 | 80.5 | 70.8 |

Table 34: Results with respect of HoF size for **mnist** dataset. ∞ means that there was no limit no limit on HoF size.

| HoF size | minimax regret ||||| adversarial accuracy ||||| computation time [s] |||||
|---|---|---|---|---|---|---|---|---|---|---|---|---|---|---|---|
| | Nash mixed tree | Top K as mixed tree | Nash single trees | Top K | Best | Nash mixed tree | Top K as mixed tree | Nash single trees | Top K | Best | Nash mixed tree | Top K as mixed tree | Nash single trees | Top K | Best |
| 0 | .226 | .220 | .220 | .220 | .221 | .966 | .966 | .966 | .966 | .966 | 15.0 | 15.0 | 15.0 | 15.0 | 15.0 |
| 10 | .200 | .206 | .210 | .209 | .215 | .975 | .974 | .972 | .970 | .968 | 15.3 | 13.8 | 14.6 | 16.0 | 13.2 |
| 20 | .200 | .206 | .208 | .208 | .213 | .976 | .974 | .973 | .971 | .969 | 15.5 | 17.0 | 14.8 | 17.9 | 14.9 |
| 50 | .198 | .203 | .207 | .207 | .212 | .977 | .975 | .973 | .972 | .970 | 20.2 | 18.9 | 16.6 | 17.1 | 17.2 |
| 100 | .196 | .201 | .205 | .205 | .210 | .979 | .977 | .975 | .973 | .970 | 20.6 | 19.9 | 19.7 | 20.8 | 17.2 |
| 200 | .196 | .200 | .203 | .203 | .207 | .982 | .979 | .977 | .974 | .971 | 23.7 | 20.3 | 22.0 | 23.4 | 22.4 |
| 500 | .196 | .199 | .203 | .202 | .207 | .982 | .980 | .978 | .974 | .971 | 26.4 | 25.8 | 22.2 | 28.3 | 21.3 |
| ∞ | .195 | .197 | .202 | .199 | .205 | .983 | .980 | .979 | .977 | .972 | 27.0 | 25.9 | 25.3 | 28.9 | 21.7 |

Table 35: Results with respect of HoF size for **f-mnist:2v5** dataset. ∞ means that there was no limit no limit on HoF size.

| HoF size | minimax regret ||||| adversarial accuracy ||||| computation time [s] |||||
|---|---|---|---|---|---|---|---|---|---|---|---|---|---|---|---|
| | Nash mixed tree | Top K as mixed tree | Nash single trees | Top K | Best | Nash mixed tree | Top K as mixed tree | Nash single trees | Top K | Best | Nash mixed tree | Top K as mixed tree | Nash single trees | Top K | Best |
| 0 | .221 | .227 | .227 | .226 | .227 | .855 | .855 | .855 | .855 | .855 | 13.9 | 13.9 | 13.9 | 13.9 | 13.9 |
| 10 | .206 | .212 | .216 | .216 | .222 | .863 | .862 | .860 | .858 | .857 | 16.0 | 15.7 | 16.1 | 15.9 | 16.3 |
| 20 | .207 | .212 | .214 | .214 | .220 | .864 | .862 | .861 | .859 | .858 | 18.8 | 16.1 | 15.1 | 18.8 | 16.4 |
| 50 | .204 | .209 | .213 | .213 | .219 | .864 | .863 | .861 | .860 | .859 | 21.7 | 20.1 | 17.9 | 18.4 | 19.4 |
| 100 | .202 | .207 | .211 | .211 | .216 | .867 | .865 | .863 | .861 | .859 | 23.3 | 19.9 | 18.5 | 21.7 | 20.4 |
| 200 | .202 | .206 | .209 | .209 | .213 | .869 | .866 | .865 | .862 | .860 | 25.2 | 23.2 | 20.1 | 24.9 | 23.4 |
| 500 | .202 | .206 | .209 | .209 | .213 | .869 | .867 | .866 | .862 | .859 | 28.6 | 24.2 | 26.1 | 27.6 | 25.5 |
| ∞ | .201 | .203 | .208 | .205 | .211 | .870 | .867 | .867 | .864 | .860 | 30.8 | 25.8 | 26.3 | 28.6 | 25.9 |

Table 36: Results with respect of HoF size for **f-mnist:3v4** dataset. ∞ means that there was no limit no limit on HoF size.

| HoF size | minimax regret ||||| adversarial accuracy ||||| computation time [s] |||||
|---|---|---|---|---|---|---|---|---|---|---|---|---|---|---|---|
| | Nash mixed tree | Top K as mixed tree | Nash single trees | Top K | Best | Nash mixed tree | Top K as mixed tree | Nash single trees | Top K | Best | Nash mixed tree | Top K as mixed tree | Nash single trees | Top K | Best |
| 0 | .227 | .233 | .233 | .233 | .234 | .854 | .854 | .854 | .854 | .854 | 15.1 | 15.1 | 15.1 | 15.1 | 15.1 |
| 10 | .213 | .218 | .222 | .222 | .229 | .862 | .861 | .859 | .857 | .856 | 18.4 | 15.9 | 14.3 | 16.3 | 16.9 |
| 20 | .213 | .218 | .220 | .221 | .226 | .863 | .861 | .860 | .858 | .857 | 18.1 | 17.1 | 17.9 | 19.7 | 17.5 |
| 50 | .210 | .216 | .220 | .220 | .225 | .863 | .862 | .860 | .859 | .858 | 23.0 | 19.2 | 16.8 | 21.5 | 18.1 |
| 100 | .208 | .213 | .217 | .217 | .223 | .866 | .864 | .862 | .860 | .858 | 22.2 | 21.0 | 18.0 | 22.5 | 18.2 |
| 200 | .208 | .212 | .215 | .215 | .220 | .868 | .865 | .864 | .861 | .859 | 25.9 | 22.1 | 22.6 | 26.5 | 24.3 |
| 500 | .208 | .212 | .215 | .215 | .220 | .868 | .866 | .865 | .861 | .858 | 31.2 | 24.9 | 24.4 | 29.3 | 22.6 |
| ∞ | .207 | .209 | .215 | .211 | .217 | .869 | .866 | .866 | .863 | .859 | 31.5 | 28.7 | 25.2 | 30.5 | 23.0 |

Table 37: Results with respect of HoF size for **f-mnist:7v9** dataset. ∞ means that there was no limit no limit on HoF size.

| HoF size | minimax regret | | | | | adversarial accuracy | | | | | computation time [s] | | | | |
|---|---|---|---|---|---|---|---|---|---|---|---|---|---|---|---|
| | Nash mixed tree | Top K as mixed tree | Nash single trees | Top K | Best | Nash mixed tree | Top K as mixed tree | Nash single trees | Top K | Best | Nash mixed tree | Top K as mixed tree | Nash single trees | Top K | Best |
| 0 | .275 | .275 | .275 | .274 | .275 | .665 | .665 | .665 | .665 | .665 | 81.1 | 81.1 | 81.1 | 81.1 | 81.1 |
| 10 | .246 | .254 | .256 | .264 | .274 | .680 | .676 | .676 | .670 | .669 | 91.4 | 84.2 | 81.6 | 87.6 | 79.7 |
| 20 | .246 | .254 | .254 | .262 | .270 | .681 | .677 | .676 | .670 | .670 | 95.9 | 91.0 | 97.2 | 97.0 | 81.4 |
| 50 | .244 | .252 | .253 | .262 | .269 | .681 | .677 | .677 | .672 | .670 | 116.2 | 102.4 | 100.5 | 104.4 | 91.7 |
| 100 | .242 | .248 | .250 | .258 | .266 | .683 | .679 | .678 | .672 | .670 | 125.8 | 110.7 | 114.8 | 123.3 | 90.1 |
| 200 | .241 | .246 | .249 | .255 | .263 | .685 | .680 | .680 | .671 | .672 | 132.6 | 124.1 | 139.0 | 132.1 | 102.7 |
| 500 | .243 | .248 | .250 | .256 | .265 | .684 | .681 | .679 | .671 | .672 | 154.2 | 144.5 | 173.1 | 163.3 | 104.6 |
| ∞ | .239 | .242 | .249 | .247 | .261 | .686 | .683 | .682 | .678 | .674 | 159.0 | 143.8 | 167.0 | 169.4 | 104.2 |

Table 38: Results with respect of HoF size for **cifar10:0v5** dataset. ∞ means that there was no limit no limit on HoF size.

| HoF size | minimax regret | | | | | adversarial accuracy | | | | | computation time [s] | | | | |
|---|---|---|---|---|---|---|---|---|---|---|---|---|---|---|---|
| | Nash mixed tree | Top K as mixed tree | Nash single trees | Top K | Best | Nash mixed tree | Top K as mixed tree | Nash single trees | Top K | Best | Nash mixed tree | Top K as mixed tree | Nash single trees | Top K | Best |
| 0 | .320 | .324 | .324 | .324 | .325 | .681 | .681 | .681 | .681 | .681 | 68.8 | 68.8 | 68.8 | 68.8 | 68.8 |
| 10 | .295 | .303 | .309 | .309 | .318 | .687 | .686 | .685 | .683 | .682 | 86.1 | 83.7 | 76.6 | 74.8 | 81.7 |
| 20 | .296 | .303 | .306 | .307 | .314 | .688 | .687 | .686 | .684 | .683 | 83.8 | 82.0 | 89.3 | 92.1 | 78.2 |
| 50 | .292 | .299 | .305 | .305 | .313 | .688 | .687 | .686 | .685 | .684 | 104.3 | 92.4 | 85.4 | 99.4 | 88.1 |
| 100 | .290 | .296 | .302 | .302 | .310 | .690 | .689 | .687 | .686 | .684 | 110.0 | 101.6 | 101.0 | 117.5 | 91.8 |
| 200 | .289 | .295 | .299 | .299 | .305 | .692 | .690 | .689 | .686 | .684 | 126.0 | 125.9 | 101.7 | 137.4 | 106.7 |
| 500 | .288 | .294 | .299 | .298 | .305 | .692 | .690 | .689 | .687 | .684 | 135.1 | 140.4 | 118.3 | 147.5 | 123.6 |
| ∞ | .287 | .290 | .298 | .293 | .302 | .693 | .691 | .690 | .688 | .685 | 151.8 | 142.6 | 124.4 | 162.6 | 125.6 |

Table 39: Results with respect of HoF size for **cifar10:0v6** dataset. ∞ means that there was no limit no limit on HoF size.

| HoF size | minimax regret | | | | | adversarial accuracy | | | | | computation time [s] | | | | |
|---|---|---|---|---|---|---|---|---|---|---|---|---|---|---|---|
| | Nash mixed tree | Top K as mixed tree | Nash single trees | Top K | Best | Nash mixed tree | Top K as mixed tree | Nash single trees | Top K | Best | Nash mixed tree | Top K as mixed tree | Nash single trees | Top K | Best |
| 0 | .305 | .318 | .318 | .317 | .318 | .652 | .652 | .652 | .652 | .652 | 62.7 | 62.7 | 62.7 | 62.7 | 62.7 |
| 10 | .289 | .297 | .303 | .302 | .311 | .659 | .657 | .656 | .655 | .654 | 67.5 | 74.9 | 60.4 | 68.0 | 69.0 |
| 20 | .289 | .297 | .300 | .300 | .308 | .659 | .658 | .657 | .655 | .655 | 71.7 | 74.3 | 76.1 | 83.4 | 69.5 |
| 50 | .286 | .293 | .299 | .299 | .306 | .660 | .658 | .657 | .656 | .655 | 86.6 | 74.9 | 84.6 | 88.4 | 76.9 |
| 100 | .284 | .290 | .296 | .296 | .303 | .661 | .660 | .659 | .657 | .655 | 108.6 | 98.3 | 78.0 | 10.3 | 86.0 |
| 200 | .283 | .288 | .293 | .292 | .299 | .663 | .661 | .660 | .657 | .656 | 111.4 | 95.0 | 93.1 | 125.3 | 93.6 |
| 500 | .282 | .288 | .293 | .292 | .299 | .663 | .661 | .660 | .658 | .655 | 120.0 | 119.8 | 102.6 | 136.5 | 109.2 |
| ∞ | .281 | .284 | .292 | .287 | .295 | .664 | .662 | .661 | .660 | .656 | 123.8 | 120.5 | 113.2 | 139.7 | 112.1 |

Table 40: Results with respect of HoF size for **cifar10:4v8** dataset. ∞ means that there was no limit no limit on HoF size.



\end{document}